\documentclass[pdflatex,sn-mathphys-num]{sn-jnl}

\usepackage{graphicx}%
\usepackage{multirow}%
\usepackage{amsmath,amssymb,amsfonts}%
\usepackage{amsthm}%
\usepackage{mathrsfs}%
\usepackage[title]{appendix}%
\usepackage{xcolor}%
\usepackage{textcomp}%
\usepackage{manyfoot}%
\usepackage{booktabs}%
\usepackage{tikz}
\usepackage{pgfplots}

\pgfplotsset{compat=1.16}
\usetikzlibrary{arrows.meta,decorations.pathmorphing,backgrounds,positioning,calc,patterns,fit}

\definecolor{holo}{RGB}{30,110,190}
\definecolor{antiholo}{RGB}{200,80,30}
\definecolor{probeC}{RGB}{40,140,90}
\definecolor{basinC}{RGB}{90,150,200}
\definecolor{badC}{RGB}{200,120,120}
\definecolor{axisgray}{RGB}{110,110,110}
\definecolor{hebbC}{RGB}{120,60,160}

\newcommand{\stagestrip}[2]{%
  \foreach \i/\txt in {1/{Energy \& drive},2/{Descent},3/{Nudge \& probe},4/{Lock-in}}{%
    \ifnum\i=#1\relax
      \node[draw=holo, fill=holo!15, rounded corners=2pt, font=\scriptsize\bfseries, minimum width=2.7cm, minimum height=0.42cm, inner sep=1pt] at ({(\i-2.5)*2.95}, #2) {\txt};
    \else
      \node[draw=axisgray!60, text=axisgray, rounded corners=2pt, font=\scriptsize, minimum width=2.7cm, minimum height=0.42cm, inner sep=1pt] at ({(\i-2.5)*2.95}, #2) {\txt};
    \fi
  }%
  \foreach \i in {1,2,3}{\draw[->, axisgray!60] ({(\i-2.5)*2.95+1.35}, #2) -- ({(\i-1.5)*2.95-1.35}, #2);}%
}

\newcommand{\R}{\mathbb{R}}
\newcommand{\C}{\mathbb{C}}
\newcommand{\T}{\mathbb{T}}
\newcommand{\im}{\mathrm{i}}
\newcommand{\dd}{\mathrm{d}}
\newcommand{\dt}{\Delta t}
\newcommand{\unitproject}{\mathrm{u}}
\DeclareMathOperator{\Real}{Re}
\DeclareMathOperator{\Imag}{Im}

\newcommand{\softproj}{\mathrm{u}_\varepsilon}
\newcommand{\ff}{\mathrm{ff}}
\newcommand{\radius}{\varrho}
\newcommand{\ip}[2]{\left\langle #1, #2 \right\rangle}

\theoremstyle{thmstyleone}%
\newtheorem{theorem}{Theorem}
\newtheorem{proposition}[theorem]{Proposition}%

\theoremstyle{thmstyletwo}%
\newtheorem{remark}{Remark}%

\theoremstyle{thmstylethree}%

\theoremstyle{thmstyleone}%
\newtheorem{lemma}[theorem]{Lemma}
\newtheorem{corollary}[theorem]{Corollary}

\begin{document}

\title[Lock-in EP: An In-Situ Training Algorithm for Oscillatory Hardware]{Lock-in EP: An In-Situ Training Algorithm for Oscillatory Hardware}


\author[1]{\fnm{Sowjanya} \sur{Tammali}}

\author*[2]{\fnm{Wilkie} \sur{Olin-Ammentorp}}\email{wolinammentorp@anl.gov}

\affil[1]{\orgdiv{Computer Science}, \orgname{Missouri University of Science and Technology}, \orgaddress{\street{500 West 15th Street}, \city{Rolla}, \postcode{65409}, \state{MO}, \country{USA}}}

\affil*[2]{\orgdiv{Mathematics \& Computer Science Division}, \orgname{Argonne National Laboratory}, \orgaddress{\street{9700 S. Cass Avenue}, \city{Lemont}, \postcode{60439}, \state{IL}, \country{USA}}}


\abstract{Analog hardware platforms offer the potential to reduce energy consumption over digital architectures, but in order to succeed, large-scale analog systems must also be able to operate with or recover from the variability of their components. Towards this goal, we derive and demonstrate the lock-in equilibrium propagation (LIEP) training method. LIEP provides local gradient information for each component in an oscillatory network without separate forward and backward sweeps, potentially allowing for \textit{in-situ} learning capabilities on analog oscillatory hardware platforms. We demonstrate that LIEP can be used both for \textit{ab-initio} training as well as recovering performance when pre-trained parameters are perturbed. We show that LIEP can be formulated as a three-factor update rule, and suggest that although the method is currently only validated on shallow networks, alternate architectures may allow it to extend to deep and large-scale networks addressing complex tasks.}




\maketitle

\section{Introduction}\label{sec1}

While innovation in artificial intelligence (AI) over the past decades has yielded novel tools and algorithms which can be adopted in many fields, the exploding popularity of these tools has also led to concerns over commensurate growth in the energy consumption of computing, particularly large-scale data centers \cite{siddikEnvironmentalFootprintData2021, economiststaffDataCentresImproved2024}. Analog computing methods have long been posited as a way to reduce the energetic cost of computational methods, including AI systems \cite{momeniTrainingPhysicalNeural2025, mannocciInmemoryComputingEmerging2023}. However, while this path still holds significant promise, major challenges remain before large-scale analog systems integrating novel computational primitives can compete with the digital systems deployed at-scale today.

One of these challenges is inherent in the shift from digital to analog computation: while digital systems overcome the variability of electronic devices by driving them to stable limits which represent digital states, analog systems must find ways to overcome or even utilize the variabilities present in the dynamic range used for their computations. Furthermore, while analog systems may utilize infinitesimally small amounts of power to carry out certain computations, a digital read-out of these results requires power, area, and time \cite{adc_survey}. As a result, the scale-up of many analog computing systems such as crossbars has encountered the issue that while analog computational kernels can be extremely efficient, the inefficient process of reading out and communicating these results can effectively cancel any efficiency gains \cite{xiaTechnologicalExplorationRRAM2016, amirsoleimaniMemoryVectorMatrix2020, aguirreHardwareImplementationMemristorbased2024}. As a result, we propose that the development of large-scale analog systems which can disrupt the highly successful space of digital logic must be able to both overcome or utilize the variability of its components, and avoid costly high-resolution conversions between analog and digital domains.

Previously, we have proposed that spiking networks which operate via the exchange of phases provides a rich set of dynamics that can be used to implement an enormous variety of algorithms --- including neural networks --- via sets of interacting, frequency-locked oscillators \cite{olin-ammentorpHyperdimensionalComputingProvides2023}. In this computational system, the utilization of timed spikes to represent phases allows for outputs to be decoded using simple 1-bit thresholding detectors measuring relative arrival times, rather than expensive multi-bit voltage or current encoded signals \cite{chenAnalysisDesignEnergy2016, caoDevelopmentTropicalAlgebraic2023}. However, even if such a phase-based system is capable of computing entirely via the exchange and modulation of analog signals, the challenge of establishing how these systems can overcome real-world analog variabilities remains. 

In this work, we present a concrete step forward to addressing this issue by deriving and demonstrating an \textit{in-situ} training algorithm which can be used to optimize a set of parameters in an oscillatory network given an objective function. This approach is based off the well-known equilibrium propagation algorithm, but is modified to match the primitives available for computation in an oscillatory system \cite{scellierEquilibriumPropagationBridging2017, laborieuxScalingEquilibriumPropagation2021}. This method, based on the injection of pre-defined sinusoidal signals into an oscillatory network, creates frequency sidebands which encode gradient information locally at each oscillator in the network. By decoding this information, updates can be applied locally at each oscillator without the need for a separate backpropagation step. As this method resembles the operating principles of a lock-in amplifier, this variation is termed ``lock-in equilibrium propagation'' (LIEP). 

Following the derivation of this method, we provide empirical results which demonstrate that LIEP can train shallow neural networks implemented via oscillatory systems, and provides updates which approximate standard backpropagation. We note that obtaining these desired results is dependent on several operating parameters controlling the frequency of the perturbation and oscillators, and characterize the resulting ``adiabatic'' operating regime necessary for LIEP to perform. 

We next characterize the performance of LIEP in a use case where we believe it will provide maximum utility: fine-tuning or ``repairing'' networks on analog hardware in which parameters drift from their desired values, have asymmetric updates, or are ``stuck'' in high or low states --- common failures among large analog memories \cite{kannanModelingDetectionDiagnosis2015, chenFaultModelingTesting2015}. We show that LIEP is effective at recovering the performance of networks which have been moderately perturbed from their ideal, digitally trained parameters, and establish the boundaries past which LIEP is no longer effective at recovering performance. 

Finally, we conclude with a discussion on an alternate, three-rule update formulation of LIEP which could potentially be implemented on neuromorphic platforms. We note that our experiments take place under the conditions of shallow networks, but suggest that previously established strategies for enabling the training of deep and attentional neural networks for oscillatory hardware may be effectively adopted for training with the LIEP algorithm, and leave a detailed exploration of this direction to future work. 

\section{Derivation}

As a basic understanding of equilibrium propagation (EP) is necessary in order to understand LIEP, we first provide a brief introduction of the method and its key differences from standard backpropagation. Readers familiar with EP may proceed directly to the derivation of LIEP (Section \ref{sec:LIEP}).

\subsection{Equilibrium Propagation}

Standard neural networks are ``feedforward'' systems: given a set of inputs and an internal state, a series of linear and nonlinear operations are used to calculate the output of the network. As a result, the output of layers deeper in the network have no effect on earlier layers; they are causally unconnected. Therefore in order to carry out credit assignment --- that is, to understand how to adjust the network's parameters to improve performance given its current output --- the feedforward network must be ``unrolled'' backwards to allow cause and effect to be reversed \cite{rumelhartLearningRepresentationsBackpropagating1986}. While extremely effective and relatively simple, as backpropagation is essentially a specialized case of the chain rule of differentiation, this process requires separate computational steps and full memory of the state, inputs, and outputs of the network. For this and other reasons, backpropagation is not generally accepted as a biologically realistic model of learning \cite{ororbiaBrainInspiredMachineIntelligence2023, songCanBrainBackpropagation}.

Equilibrium propagation (EP) was developed as one alternative method which can be used to train neural networks \cite{scellierEquilibriumPropagationBridging2017}. However, standard feedforward networks cannot be directly applied to EP; the core principle of EP is to define an energy function on a network of components in which influences are bidirectional. To provide a concrete example, let us consider a neural network with a state $s$ containing its neuronal activations, and a set of parameters $\theta$ containing weights $W$. A Hopfield-like scalar energy function \cite{hopfieldNeuralNetworksPhysical1982} for this network is defined as a function of the parameters and state:
\begin{equation}
E(\theta, s) = \frac{1}{2}\sum_{\ell} \|s_\ell\|^2 \;-\; \sum_{\ell} \rho(s_\ell)^\top W_\ell\, \rho(s_{\ell+1})
\end{equation}
Where the first term is a norm on its activations and the second term symmetrically couples the outputs of the layer $\ell$ and $\ell+1$ after they have been passed through the activation function $\rho$. These energy terms are summed across layers to provide a single, scalar energy function $E(\theta, s)$ for the entire network. Unlike a feedforward network, we cannot simply solve this state $s$ in one pass, as the symmetric connections between layers allow influences from one neuron to reach forward and backward across the entire network. Instead to find a state $s$ of the network given a set of inputs ``clamped'' at the first layer $\ell_0$, we iteratively solve an energy minimization equation through a simulated time:
\begin{equation}
\frac{ds}{dt} = -\frac{\partial E}{\partial s}(\theta, s)
\end{equation}
Given that the network satisfies Lyapunov conditions, it will eventually descend to a fixed point $s_0^*(\theta)$ which is defined as the ``free equilibrium.'' This energy descent allows us to find the state of the network given a set of inputs. While this descent appears laborious and unnecessary given the simple, single-step operation of feedforward networks, the crucial detail of this method is that it is assumed that a physical computing platform will compute this function implicitly via energy relaxation in continuous time, rather than simulating it step-by-step as is done in a digital computer. 

The free equilibrium state may not provide an output which matches a desired value $y$. Therefore, a scalar cost term $C(s,y)$ is introduced which measures the distance of the state from a desired values. The energy of the network is then augmented with a contribution from this cost term, multiplied by the small scalar value $\beta$: 

\begin{equation}
F(\theta, s, \beta) = E(\theta, s) + \beta\, C(s, y)
\end{equation}
With this new contribution added to the energy term, the network descends on $F$ to reach a ``nudged'' fixed point. The central result of EP proves that the gradient of a loss function $L(\theta)$ defined at the fixed free point, $L(\theta) = C(s_0^*(\theta), y)$, can be found as a difference of values found between the free and nudged states:

\begin{equation}\label{eq:eqprop_theorem}
\frac{\partial L}{\partial \theta} = \lim_{\beta \to 0} \frac{1}{\beta}\left[\frac{\partial E}{\partial \theta}\big(\theta, s_\beta^*\big) - \frac{\partial E}{\partial \theta}\big(\theta, s_0^*\big)\right]
\end{equation}
Therefore, by stepping through successive free and nudged equilibrium points, the network dynamics themselves express the information necessary to descend a loss function. 

\subsection{Lock-in Equilibrium Propagation}
\label{sec:LIEP}

Rather than a single algorithm, we may also view EP as a ``recipe'' to develop variational methods applied to other network types and domains --- for instance, holomorphic EP extends it to the case of complex-valued networks \cite{laborieuxHolomorphicEquilibriumPropagation2022}. Following this example, LIEP demonstrates that the principles of EP may also be applied to the domain of oscillators which compute by exchanging relative phase values. To do so, we must first formally define the domain of our phase network's state $z$, define a scalar energy function on this state, and show that descending this energy function will reach a stable state. Next, we define a nudging function and cost term for the network's outputs, and show that the nudged and free states can be used to extract a useful gradient to reduce a loss term. Additionally, while the following sections illuminate the most important aspects of matching these stages to derive the LIEP method, some details have been excluded in the interest of brevity; the full proof is presented in the supplementary materials. 

\subsubsection{Phase Energy Functions}

Following our prior work, we start by considering a network of frequency-locked oscillators arranged in $L$ layers, which produce a vector of complex states $z$ representing the phase of each oscillator with respect to a reference \cite{olin-ammentorpHyperdimensionalComputingProvides2023}. The state of each layer $\ell$ of $d_\ell$ oscillators lies on the unit torus in $\C^{d_\ell}$, $z_\ell \in \T^{d_\ell} = \{z \in \C^{d_\ell} : |z_j| = 1 \ \forall j\}$. To define an energy function on this state, the same cosine-similarity based metric used to calculate the ``closeness'' of two phase vectors may be applied directly as the scalar energy function for a single layer:

\begin{equation}\label{eq:phasor_layer_energy}
    E_\ell = -\Real\langle W_\ell z_{\ell-1}, z_\ell\rangle
\end{equation}
where $W_\ell$ are real-valued weights. Cosine similarity in the complex plane is calculated by taking the real component of the Hermitian inner product $\langle v, w\rangle = v^H w$ between the weighted drive from the previous layer's state ($W_\ell z_{\ell-1}$) and the layer's current state ($z_\ell$); this quantity is maximized when the drive into a layer and its oscillatory states align in phase. Similarly, we may define the cost function at the output based on the same metric:

\begin{equation}\label{eq:phasor-cost}
C(z_L, y) = 1 - \frac{1}{d}\Real\langle z_L, y \rangle,
\end{equation}
where $z_L$ is the network's output layer of $d$ oscillators and $y$ is the desired output expressed in terms of relative phases on the same domain as $z$. This allows us to produce the scalar energy function for the entire network:

\begin{equation}\label{eq:phasor-energy}
\Phi(z;\beta) = E(z) + \beta\,C(z_L,y) = -\sum_{\ell=1}^{L}\Real\langle W_\ell z_{\ell-1}, z_\ell\rangle + \beta\,C(z_L,y),
\end{equation}
where $\beta$ is the ``nudging'' term which is manipulated to move between free states and ``nudged'' states influenced by the cost function.

\subsubsection{Phase Energy Descent}

Given our energy functions, we must next define a settling function which guarantees that an initial state $z$ will descend to a stable fixed point $z^\circ$ given an input drive to the network, $x$. Given our real-valued, scalar energy functions defined in the previous section, we may define Wirtinger derivatives between our complex state and the energy function. The direction of steepest descent for this energy function given our complex state is found via the term $-2\,\partial\Phi/\partial\bar z_\ell$. To descend our energy function through successive steps in state, we first find the following derivatives:

\begin{align}
\frac{\partial \Phi}{\partial \bar z_\ell} = -\tfrac12\big(W_\ell z_{\ell-1} + W_{\ell+1}^\top z_{\ell+1}\big),\qquad \ell = 1,\dots,L-1, \label{eq:grad-hidden}\\
\frac{\partial \Phi}{\partial \bar z_L} = -\tfrac12\big(W_L z_{L-1} + \tfrac{\beta}{d}\,y\big). \label{eq:grad-output}
\end{align}
for the hidden and output layers, respectively. Doubling and negating these terms, we produce the direction of steepest descent for each layer, which we term the \textit{drive}:

\begin{equation}\label{eq:drive}
D_\ell := -2\,\frac{\partial\Phi}{\partial\bar z_\ell} = W_\ell z_{\ell-1} + W_{\ell+1}^\top z_{\ell+1},
\qquad W_{L+1}^\top z_{L+1} \equiv \tfrac{\beta}{d}\,y.
\end{equation}

We note that in the above drive equation for the output, the layer $L+1$ represents the desired output $y$, and is not physically realized in the system other than nudging the output oscillators. The resulting structure is illustrated in Figure~\ref{fig:energy-drive}.

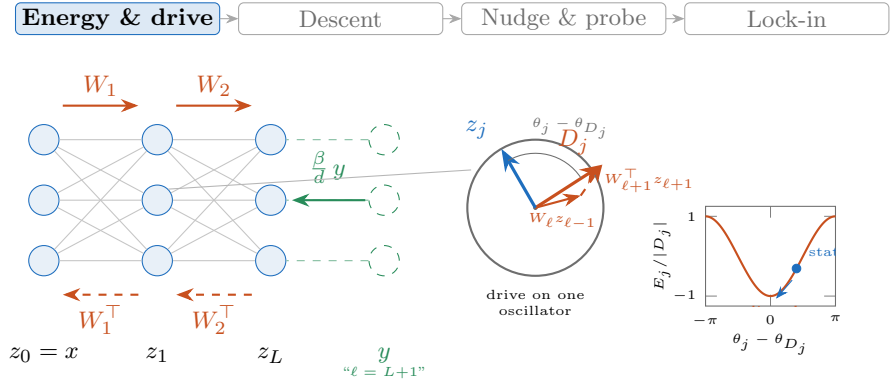
\begin{figure}[t]
\centering
\begin{tikzpicture}[>=Stealth]
\stagestrip{1}{0}
\begin{scope}[every node/.style={font=\scriptsize}]
  \foreach \c/\x in {0/-5.4, 1/-3.9, 2/-2.4}{
    \foreach \r in {1,2,3}{ \node[circle, draw=holo, fill=holo!10, minimum size=0.4cm, inner sep=0pt] (n\c\r) at (\x, {-0.8-0.8*\r}) {}; }
  }
  \foreach \r in {1,2,3}{ \node[circle, draw=probeC, dashed, minimum size=0.4cm, inner sep=0pt] (n3\r) at (-0.9, {-0.8-0.8*\r}) {}; }
  \foreach \c/\d in {0/1, 1/2}{ \foreach \r in {1,2,3}{ \foreach \q in {1,2,3}{ \draw[axisgray!40, thin] (n\c\r) -- (n\d\q); } } }
  \foreach \r in {1,2,3}{ \draw[probeC!60, thin, dashed] (n3\r) -- (n2\r); }
  \foreach \c/\d/\lab in {0/1/{W_1}, 1/2/{W_2}}{
    \draw[->, antiholo, thick] ($(n\c1)+(0.25,0.45)$) -- ($(n\d1)+(-0.25,0.45)$) node[midway, above, antiholo] {$\lab$};
    \draw[->, antiholo, thick, dashed] ($(n\d3)+(-0.25,-0.45)$) -- ($(n\c3)+(0.25,-0.45)$) node[midway, below, antiholo] {$\lab^{\top}$};
  }
  \draw[->, probeC, thick] ($(n32)+(-0.25,0)$) -- ($(n22)+(0.28,0)$);
  \node[probeC, above, font=\scriptsize] at (-1.65,-2.3) {$\tfrac{\beta}{d}\,y$};
  \node at (-5.4,-4.45) {$z_0 = x$};
  \node at (-3.9,-4.45) {$z_1$};
  \node at (-2.4,-4.45) {$z_L$};
  \node[probeC, align=center] at (-0.9,-4.55) {$y$\\[-2pt] {\tiny ``$\ell = L{+}1$''}};
\end{scope}
\begin{scope}[shift={(1.1,-2.5)}]
  \draw[axisgray!50, thin] (n12.north east) -- (-0.85,0.5);
  \draw[axisgray, thick] (0,0) circle (0.9);
  \fill[axisgray] (0,0) circle (0.03);
  \coordinate (ff) at (15:0.62);
  \coordinate (fb) at ($(ff)+(55:0.5)$);
  \draw[->, antiholo, thick] (0,0) -- (ff) node[pos=0.6, below, font=\tiny, antiholo, yshift=-1pt] {$W_\ell z_{\ell-1}$};
  \draw[->, antiholo, thick, dashed] (ff) -- (fb) node[pos=0.5, right, font=\tiny, antiholo, xshift=1pt] {$W_{\ell+1}^{\top} z_{\ell+1}$};
  \draw[->, antiholo, very thick] (0,0) -- (fb) node[pos=1.0, above left, font=\scriptsize, antiholo, xshift=-1pt] {$D_j$};
  \draw[->, holo, very thick] (0,0) -- (120:0.9) node[pos=1, anchor=south east, font=\scriptsize, holo, xshift=-1pt] {$z_j$};
  \draw[axisgray, thin] (120:0.72) arc (120:33:0.72);
  \node[font=\tiny, axisgray] at (66:1.15) {$\theta_j-\theta_{D_j}$};
  \node[font=\tiny, align=center] at (0,-1.25) {drive on one\\oscillator};
\end{scope}
\begin{axis}[at={(3.35cm,-3.8cm)}, anchor=south west, width=3.3cm, height=2.9cm,
  xmin=-180, xmax=180, ymin=-1.25, ymax=1.25, xtick={-180,0,180}, xticklabels={$-\pi$,$0$,$\pi$}, ytick={-1,1},
  xlabel={$\theta_j-\theta_{D_j}$}, ylabel={$E_j/|D_j|$}, tick label style={font=\tiny}, label style={font=\tiny},
  xlabel style={yshift=4pt}, ylabel style={yshift=-6pt}, axis line style={axisgray}]
  \addplot[antiholo, thick, domain=-180:180, samples=80] {-cos(x)};
  \addplot[holo, only marks, mark=*, mark size=1.6pt] coordinates {(72,{-cos(72)})};
  \node[font=\tiny, holo, anchor=west] at (axis cs:80,0.1) {state};
  \node[font=\tiny, antiholo, anchor=north] at (axis cs:0,-1.0) {aligned};
  \draw[->, holo, thin] (axis cs:60,{-cos(60)-0.1}) -- (axis cs:15,{-cos(15)-0.1});
\end{axis}
\end{tikzpicture}
\caption{Layers of neurons with states $z_\ell$ are connected via symmetric weights (left). The energy of each layer is defined as the non-alignment in phases between that layer's current state and its input (Eq. \ref{eq:phasor_layer_energy}). For intermediate layers (e.g. $z_1$), descending this derived energy function can be accomplished by following the ``drive'' function (Eq. \ref{eq:drive}) which is dependent on both the previous layer ($z_{\ell-1}$) and next layer ($z_{\ell+1}$). The phase term of the state $\theta_j$ aligns with the phase of the drive term $\theta_{D_j}$, decreasing the layer's energy over time. The desired outputs $y$ can influence the energy at each layer by contributing an output layer drive from a ``virtual'' layer ($L+1$), normalized by the output dimensionality $d$ and modulated by $\beta$.}
\label{fig:energy-drive}
\end{figure}
 Following the drive at each layer produces a continuous-time gradient flow in $\C^{d_\ell}$ for each oscillator:

\begin{equation}
\frac{\dd z_\ell}{\dd t} = -\gamma z_\ell - 2\frac{\partial\Phi}{\partial\bar z_\ell} = -\gamma z_\ell + D_\ell,\qquad \ell = 1,\dots,L. \label{eq:continuous-flow}
\end{equation}

where $\gamma$ is a linear decay term with $\gamma > 0$. However, to meet our prior requirement that fixed points of the network lie on the unit torus, we must apply this drive to ensure that our energy descent approaches and ``lands'' on this subspace. Therefore, a discretized energy descent utilizing the drive projects updates towards the unit torus with each step $\Delta t$:

\begin{equation}
    z_\ell \leftarrow (1 - \dt)\, z_\ell + \dt\; \unitproject(D_\ell),
    \label{eq:discrete-settle}
\end{equation}

where $\unitproject(g) = g/|g|$ is the unit projection, applied to each oscillator's drive, and $\gamma = 1$. This yields an Euler update step for the oscillatory system to descend to a stable fixed point $z^\circ$ on the unit torus given an initial $z$ and a drive function $D$. This movement towards fixed points aligns the drive of each layer and its oscillators in phase, reducing the energy of the system over time. The rate at which the system converges towards a fixed point is set by this linear update, and we denote this measurable settling rate with the value $\kappa$. Additionally, this projection onto the torus takes the place of a saturating nonlinearity in a standard, real-valued network; rather than applying a function like $\rho(s)$, our update is projected using this update step. Figure~\ref{fig:descent} illustrates the projection and the settling trajectory.

\begin{figure}[t]
\centering
\begin{tikzpicture}[>=Stealth]
\stagestrip{2}{0}
\draw[draw=axisgray!55, fill=axisgray!4, rounded corners=3pt] (-6.70,-1.00) rectangle (-3.75,-3.70);
\node[font=\scriptsize\bfseries, text=axisgray, anchor=north, align=center, text width=2.8cm]
  at (-5.225,-1.08) {traditional\\feedforward network};
\begin{axis}[scale only axis, at={(-6.10cm,-3.45cm)}, anchor=south west, width=2.25cm, height=1.35cm,
  xmin=-2.6, xmax=2.6, ymin=-1.5, ymax=1.5, xtick=\empty, ytick=\empty,
  axis lines=middle, axis line style={axisgray}, clip=false]
  \draw[axisgray!60, dashed] (axis cs:-2.6,1) -- (axis cs:2.6,1);
  \draw[axisgray!60, dashed] (axis cs:-2.6,-1) -- (axis cs:2.6,-1);
  \addplot[holo, thick, domain=-2.6:2.6, samples=60] {tanh(x)};
  \node[font=\tiny, axisgray, anchor=west] at (axis cs:2.65,1) {$+1$};
  \node[font=\tiny, axisgray, anchor=west] at (axis cs:2.65,-1) {$-1$};
  \node[font=\tiny, anchor=north east] at (axis cs:2.6,-0.1) {$s$};
  \node[font=\tiny, holo, anchor=west] at (axis cs:-2.45,0.42) {$\rho(s)$};
\end{axis}
\node[font=\scriptsize, align=center, anchor=north, text width=3.0cm] at (-5.225,-4.05)
  {bounded by a squashing\\nonlinearity $\rho$};
\begin{scope}[shift={(-1.05,-2.35)}]
  \draw[axisgray, thick] (0,0) circle (1.05);
  \fill[axisgray] (0,0) circle (0.03);
  \draw[axisgray!40, thin] (0,0) -- (35:1.45);
  \draw[->, antiholo, thick] (0,0) -- (35:0.52);
  \node[font=\tiny, antiholo, anchor=north west, inner sep=1pt] at (35:0.44) {$D_1$};
  \draw[->, holo!75, thin, dashed, shorten >=2.5pt] (35:0.57) -- (35:1.05);
  \fill[holo] (35:1.05) circle (0.055);
  \node[font=\tiny, holo, anchor=west, inner sep=2.5pt] at (35:1.09) {$\unitproject(D_1)$};
  \draw[axisgray!40, thin] (0,0) -- (130:1.72);
  \draw[->, antiholo, thick] (0,0) -- (130:1.50);
  \node[font=\tiny, antiholo, anchor=south west, inner sep=1pt] at (130:1.50) {$D_2$};
  \draw[->, holo!75, thin, dashed, shorten >=2.5pt] (130:1.45) -- (130:1.05);
  \fill[holo] (130:1.05) circle (0.055);
  \node[font=\tiny, holo, anchor=east, inner sep=2.5pt] at (-1.30,0.82) {$\unitproject(D_2)$};
  \draw[axisgray!40, thin] (0,0) -- (225:1.45);
  \draw[->, antiholo, thick] (0,0) -- (225:0.85);
  \node[font=\tiny, antiholo, anchor=north west, inner sep=1pt] at (225:0.55) {$D_3$};
  \draw[->, holo!75, thin, dashed, shorten >=2.5pt] (225:0.90) -- (225:1.05);
  \fill[holo] (225:1.05) circle (0.055);
  \node[font=\tiny, holo, anchor=north east, inner sep=2.5pt] at (225:1.09) {$\unitproject(D_3)$};
\end{scope}
\node[font=\scriptsize, align=center, anchor=north, text width=3.8cm] at (-1.15,-4.05)
  {a drive of any magnitude\\maps to one phase:\\$\unitproject(D) = D/|D|$ on $|z| = 1$};
\begin{scope}[shift={(3.35,-2.35)}]
  \draw[axisgray, thick] (0,0) circle (1.5);
  \draw[axisgray!40, thin] (-1.8,0) -- (1.8,0);
  \draw[axisgray!40, thin] (0,-1.8) -- (0,1.8);
  \begin{scope}[scale=0.9375]
  \draw[holo, thick] plot[smooth] coordinates {(0,0) (-0.049,0.558) (0.157,0.890) (0.434,1.030) (0.694,1.048) (0.909,1.004) (1.075,0.934) (1.199,0.859) (1.288,0.791) (1.353,0.734) (1.398,0.687) (1.430,0.651) (1.453,0.623) (1.468,0.602) (1.479,0.587) (1.487,0.575)};
  \foreach \pt in {(-0.049,0.558), (0.157,0.890), (0.434,1.030), (0.694,1.048), (0.909,1.004), (1.075,0.934), (1.199,0.859), (1.288,0.791), (1.353,0.734), (1.398,0.687), (1.430,0.651), (1.453,0.623), (1.468,0.602), (1.479,0.587), (1.487,0.575)}{ \fill[holo] \pt circle (0.043); }
  \end{scope}
  \fill[holo] (0,0) circle (0.05);
  \node[font=\tiny, holo, anchor=north east] at (0,-0.05) {$k=0$};
  \draw[->, antiholo, thick] (0,0) -- (20:1.5);
  \node[font=\tiny, antiholo, anchor=south west] at (20:1.52) {$\unitproject(D)$};
  \node[font=\tiny, holo, anchor=west] at (45:1.58) {$z^{(k)}$};
\end{scope}
\node[font=\scriptsize, align=center, anchor=north, text width=5.0cm] at (3.35,-4.05)
  {settle: $z \leftarrow (1-\dt)z + \dt\,\unitproject(D)$\\[1pt]
   once aligned, $1-|z|$ decays at rate $1-\dt$};
\end{tikzpicture}
\caption{In standard neural networks, the state for a layer is reached by multiplying inputs by weights and applying a nonlinear ``squashing'' activation function $\rho$ (left). In EP networks, fixed states are reached by following the drive function $D$ at each layer (Eq. \ref{eq:drive}). In LIEP networks, each layer's oscillatory state $z_\ell$ lies on the unit torus in the complex plane, where each point corresponds to set of phase values. To meet this prerequisite, the projection function $u$ maps the drive ($D_1$, $D_2$, $D_3$) onto the torus (center). By utilizing a linear update of scaled drive steps (Eq. \ref{eq:discrete-settle}), the layer's state converges to a fixed point on the unit torus over time (right).}
\label{fig:descent}
\end{figure}
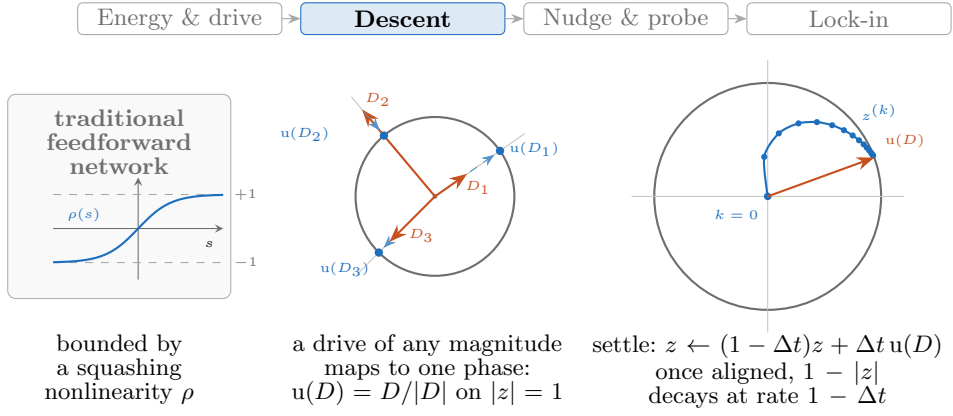

\subsubsection{Nudged Phases}

We have introduced the cost function $C(z_L,y)$ and nudging term $\beta$, but have not yet established how the resulting free and nudged stable states can be applied to produce useful update gradients; we approach this task next. 

Given an input, output, and weights, we may view the relationship between the nudging term $\beta$ and the resulting fixed point as a map $\beta \mapsto z^*(\beta)$. Therefore the derivative $\dd z^*/\dd\beta$ can be estimated by symmetrically sweeping $\beta$ above and then below the free point, allowing three fixed states to be measured: $z^\circ$, $z_{\beta +}^*$, and $z_{\beta -}^*$. Using these three estimates allows the gradient around the center point to be measured with reduced bias \cite{laborieuxScalingEquilibriumPropagation2021}. If we represent this as a temporal process of measuring three separate states, we may control $\beta$ following a square wave:

\begin{equation}
    \beta_{\text{square}}(t) = \varepsilon\,\mathrm{sgn}\cos(\omega_p t)
\end{equation}

With this square wave probe controlling $\beta$, we can extract the centered difference around $z^\circ$ by multiplying the resulting $z(t)$ by the square-wave probe and averaging the result. However, to avoid the high-frequency system harmonics and long settling times which are introduced by such a square-wave probe, we may instead investigate the case where we use only the fundamental of the square wave in order to control $\beta$:

\begin{equation}
    \beta(t) = \varepsilon\cos(\omega_p t)
\end{equation}

The cosine probe alternates between the positive and negative nudges symmetrically, and excites the network at a single frequency $\omega_p$ which can be similarly demodulated and averaged to recover the central difference estimate of our gradient:

\begin{equation}
z(t) - z^\circ \approx \varepsilon\,\frac{\dd z^*}{\dd\beta}\Big|_{\beta=0}\cos(\omega_p t).
\end{equation}
provided that the probe frequency $\omega_p$ is much less than the settling rate $\kappa$. We will expand on this condition and its implications below (Sec. \ref{sec:adiabatic_regime}).

This estimate also rejects noise and drift at low and high frequencies, similar to the ``lock-in'' measurement technique which gives it its name. Thus, by first reaching a stable, fixed free point $z^\circ$ and then injecting the cost term into our energy with a temporally-varying magnitude controlled by $\beta(t) = \varepsilon\cos(\omega_p t)$, the resulting difference between the free and nudged states provides the information necessary to reduce the loss of the network. Figure~\ref{fig:probes} compares the four nudging protocols.

\begin{figure}[t]
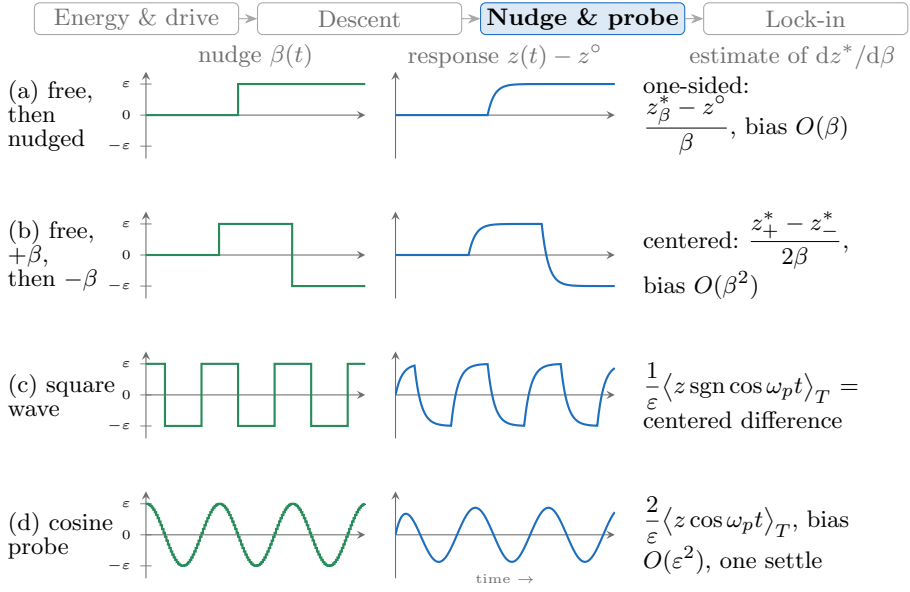

\centering
\include{figures/liep_3_tikz}
\caption{Utilizing the network's energy function, a drive function combined with an appropriate update step allows the network to reach stable, fixed points given an input $x$, desired output $y$, and cost weighting term $\beta$. When $\beta=0$, the network's ``free'' fixed point $z^\circ$ is generated. Varying $\beta$ influences the network's energy, producing different fixed points when $y$ does not align with the network's output. Varying $\beta$ from 0 to a positive value $\varepsilon$ is simple (a), but produces a biased estimate of the local gradient around the free fixed point. Adding a second sampling point at $-\varepsilon$ reduces the error of this estimate (b). Alternatively, we may modulate $\beta$ through time following a square-wave schedule to sweep it through the necessary values (c). We may further simplify the probe by replacing the square wave with a simple sinusoid which sweeps $\beta$ through time (d). This sinusoidal schedule replaces discrete ``settling'' stages, instead injecting the probe as a time-varying signal with a single harmonic component and allowing the local gradient at the free fixed point to be demodulated.}
\label{fig:probes}
\end{figure}

\subsubsection{Loss Descent}

To close the training loop, we must now translate the information available to us by nudging the system into a reduction of its loss function, defined as $\mathcal{L}(W) := C(z^\circ_L(W), y)$: the cost of the network at the free point given input, output, and a set of weights. To establish this connection, we establish the formal relation between the difference in free and nudged phases shown above ($z(t) - z^\circ$) and the gradient of loss with respect to weights ($\partial \mathcal{L} / \partial W$). First, we revisit our energy function (Eqn. \ref{eq:phasor-energy}) and derive it with respect to weights: 

\begin{equation}
\frac{\partial \Phi}{\partial W_\ell} = -\Real\big( z_\ell z_{\ell-1}^H \big) \;\equiv\; -H_\ell.
\end{equation}
where $(\cdot)^H$ is the conjugate transpose. $H_\ell$ represents the Hebbian product for each layer, which we may interpret as the similarity between each pre-synaptic state $z_{\ell-1}$ and post-synaptic state $z_\ell$. When the value $\beta$ controlling the contribution of the cost function to the network's total energy is varied following the cosine probe, the states of the network oscillate, with:

\begin{equation}
    z_\ell(t) \approx z^\circ_\ell + \delta z_\ell(t)
\end{equation}
where $\delta z_\ell(t)$ is proportional to the amplitude of the cosine probe $\varepsilon$ and oscillates at $\omega_p$. Substituting this into our Hebbian per-layer product, we obtain:

\begin{equation}
H_\ell(t) = \Real\big( z_\ell(t) z_{\ell-1}(t)^H \big)
\approx H^\circ_\ell + \Real\big( \delta z_\ell(t) \big(z^\circ_{\ell-1}\big)^{H} + z^\circ_\ell\, \delta z_{\ell-1}(t)^H \big).
\end{equation}
Thus, when the $\beta$ term is varied following the cosine probe, the Hebbian alignment of the entire network fluctuates at the frequency of the probe $\omega_p$ and with an amplitude controlled by the derivative of the Hebbian with respect to the nudge ($\varepsilon\, \dd H_\ell/\dd\beta$). To demodulate this quantity, the per-layer Hebbian values are multiplied by another cosine probe and averaged over several sampling periods:

\begin{equation}
\langle f \rangle_T := \frac{1}{T}\int_0^T f(t)\,\dd t, \qquad T = \frac{2\pi n}{\omega_p},
\qquad
\big\langle H_\ell\cos(\omega_p t)\big\rangle_T = \frac{\varepsilon}{2}\frac{\dd H_\ell}{\dd\beta} + O(\varepsilon^3).
\end{equation}

By the equilibrium propagation theorem (Eqn. \ref{eq:eqprop_theorem}), we may now solve for the gradient of the loss with respect to the weights at $\beta$=0; combining this with the demodulated Hebbian gives the per-layer update:

\begin{equation}\boxed{\;
\frac{\partial \mathcal{L}}{\partial W_\ell}
= -\frac{\dd H_\ell}{\dd \beta}\Big|_{\beta = 0}
= -\frac{2}{\varepsilon}\,\Big\langle \Real\big(z_\ell z_{\ell-1}^H\big)\cos(\omega_p t) \Big\rangle_T
+ O(\varepsilon^2)
\;}\label{eq:lockin-gradient}
\end{equation}

This final formula enables the gradient of the network's loss --- its cost at the free point --- with respect to weights to be measured. As a result of applying the lock-in method, this estimate is additionally free of bias from the free-state Hebbian and even-order harmonics, which average to zero over the sampling period. The demodulation chain is illustrated in Figure~\ref{fig:lockin}.

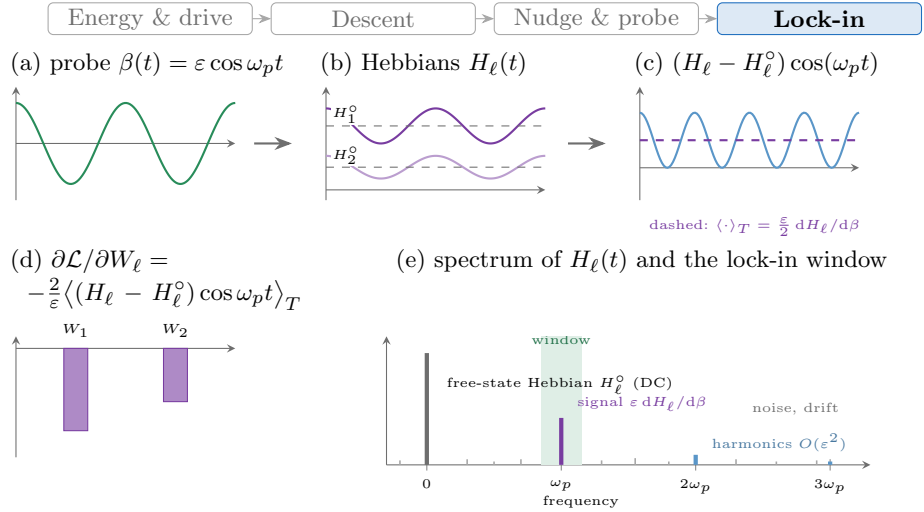
\begin{figure}[t]
\centering
\pgfplotsset{lockinpanel/.style={scale only axis, width=2.9cm, height=1.5cm, xmin=0, xmax=720, ytick=\empty, xtick=\empty,
                axis line style={axisgray}, axis x line=middle, axis y line=left, tick label style={font=\tiny}, label style={font=\tiny}}}
\begin{tikzpicture}[>=Stealth]
\stagestrip{4}{0}
\node[font=\scriptsize, anchor=west] at (-6.4,-0.6) {(a) probe $\beta(t) = \varepsilon\cos\omega_p t$};
\begin{axis}[lockinpanel, at={(-6.2cm,-0.9cm)}, anchor=north west, ymin=-1.4, ymax=1.4]
  \addplot[probeC, thick, domain=0:720, samples=120] {cos(x)};
\end{axis}
\node[font=\scriptsize, anchor=west] at (-2.3,-0.6) {(b) Hebbians $H_\ell(t)$};
\begin{axis}[lockinpanel, at={(-2.1cm,-0.9cm)}, anchor=north west, ymin=-0.1, ymax=1.0]
  \addplot[hebbC, thick, domain=0:720, samples=120] {0.62+0.17*cos(x)};
  \addplot[hebbC!50, thick, domain=0:720, samples=120] {0.22+0.11*cos(x)};
  \draw[axisgray, dashed] (axis cs:0,0.62) -- (axis cs:720,0.62);
  \draw[axisgray, dashed] (axis cs:0,0.22) -- (axis cs:720,0.22);
  \node[font=\tiny, anchor=south west, fill=white, inner sep=0.5pt] at (axis cs:15,0.63) {$H^\circ_1$};
  \node[font=\tiny, anchor=south west, fill=white, inner sep=0.5pt] at (axis cs:15,0.23) {$H^\circ_2$};
\end{axis}
\node[font=\scriptsize, anchor=west] at (1.85,-0.6) {(c) $(H_\ell - H^\circ_\ell)\cos(\omega_p t)$};
\begin{axis}[lockinpanel, at={(2.05cm,-0.9cm)}, anchor=north west, ymin=-0.1, ymax=0.25]
  \addplot[basinC, thick, domain=0:720, samples=120] {0.17*cos(x)*cos(x)};
  \draw[hebbC, dashed, thick] (axis cs:0,0.085) -- (axis cs:720,0.085);
\end{axis}
\node[font=\tiny, hebbC, anchor=west] at (2.05,-2.75) {dashed: $\langle\cdot\rangle_T = \tfrac{\varepsilon}{2}\,\dd H_\ell/\dd\beta$};
\draw[->, thick, axisgray] (-3.05,-1.65) -- (-2.55,-1.65);
\draw[->, thick, axisgray] (1.1,-1.65) -- (1.6,-1.65);
\node[font=\scriptsize, anchor=north west, align=left, text width=4.7cm] at (-6.4,-2.92)
  {(d) $\partial\mathcal{L}/\partial W_\ell =$\\[1pt] \hspace*{2.2mm}$-\tfrac{2}{\varepsilon}\big\langle (H_\ell - H^\circ_\ell)\cos\omega_p t\big\rangle_T$};
\begin{axis}[scale only axis, width=2.9cm, height=1.6cm, at={(-6.2cm,-4.2cm)}, anchor=north west,
  ybar, bar width=9pt, xmin=0.4, xmax=2.6, ymin=-0.45, ymax=0.05, xtick={1,2}, xticklabels={$W_1$,$W_2$}, ytick=\empty,
  axis line style={axisgray}, axis x line=middle, axis y line=left, tick label style={font=\tiny},
  xticklabel style={anchor=south, yshift=3pt}, xtick style={draw=none}]
  \addplot[fill=hebbC!60, draw=hebbC] coordinates {(1,-0.34) (2,-0.22)};
\end{axis}
\node[font=\scriptsize, anchor=north west] at (-1.3,-2.92) {(e) spectrum of $H_\ell(t)$ and the lock-in window};
\begin{axis}[scale only axis, width=6.4cm, height=1.7cm, at={(-1.3cm,-4.2cm)}, anchor=north west,
  xmin=-0.3, xmax=3.3, ymin=0, ymax=1.15, xtick={0,1,2,3}, xticklabels={$0$,$\omega_p$,$2\omega_p$,$3\omega_p$}, ytick=\empty,
  axis line style={axisgray}, axis x line=bottom, axis y line=left, tick label style={font=\tiny}, xlabel={frequency}, label style={font=\tiny}, xlabel style={yshift=6pt, anchor=north east}]
  \fill[probeC!15] (axis cs:0.85,0) rectangle (axis cs:1.15,1.15);
  \addplot[axisgray!60, ycomb, mark=none, line width=0.6pt] coordinates {(-0.2,0.04) (0.3,0.05) (0.5,0.03) (0.7,0.05) (1.3,0.04) (1.55,0.06) (1.8,0.03) (2.3,0.05) (2.6,0.03) (2.9,0.04)};
  \addplot[axisgray, ycomb, mark=none, line width=1.6pt] coordinates {(0,1.0)};
  \addplot[hebbC, ycomb, mark=none, line width=1.6pt] coordinates {(1,0.42)};
  \addplot[basinC, ycomb, mark=none, line width=1.6pt] coordinates {(2,0.09) (3,0.03)};
  \node[font=\tiny, anchor=west] at (axis cs:0.08,0.72) {free-state Hebbian $H^\circ_\ell$ (DC)};
  \node[font=\tiny, hebbC, anchor=west] at (axis cs:1.05,0.55) {signal $\varepsilon\,\dd H_\ell/\dd\beta$};
  \node[font=\tiny, basinC!80!black, anchor=west] at (axis cs:2.05,0.2) {harmonics $O(\varepsilon^2)$};
  \node[font=\tiny, axisgray, anchor=west] at (axis cs:2.35,0.5) {noise, drift};
  \node[font=\tiny, probeC!80!black, anchor=south] at (axis cs:1.0,1.0) {window};
\end{axis}
\end{tikzpicture}
\caption{Applying a time-varying cosine probe $\beta(t)$ continuously modulates the contribution of the cost function to the network's energy function (a). As a result, the Hebbian energy term of each layer $H_\ell(t)$ varies through time as well (b). Removing the ``DC'' energy term of the network at its free fixed point $H^\circ_\ell$ and multiplying by the probe to ``lock-in,'' the resulting signal (c) can be averaged over multiple periods of the probe to estimate the gradient loss function w.r.t. network weights (Eq. \ref{eq:lockin-gradient}) at each network layer (d). By sampling full periods, the explicit subtraction of the DC signal may also be dropped, as it averages to zero. Inspecting the spectrum of the energy $H_\ell$, the ``DC'' free state sits at zero, with the cosine probe creating a clear sideband at its fundamental frequency $\omega_p$ and higher-order harmonics (e).}
\label{fig:lockin}
\end{figure}

By applying standard methods such as stochastic gradient descent with the gradient estimates provided by LIEP, the phase-based network can be trained to produce specific output phases given a set of inputs and weights. This ability in itself is not novel, as other standard backpropagation-based methods have already established this capability. However, as far as the authors are aware, this is the first published method for training continuously-valued phase networks which utilizes frequency-domain techniques to avoid discrete settling stages \cite{martinEqSpikeSpikedrivenEquilibrium2021, delacourLagrangeOscillatoryNeural2025, laydevantTrainingIsingMachine2024}. A high-level overview of the method is presented in Figure \ref{fig:roadmap}.

\begin{figure}[t]
\centering
\begin{tikzpicture}[>=Stealth,
  stage/.style={draw=holo, fill=holo!8, rounded corners=3pt, text width=2.45cm, minimum height=2.1cm, align=center, inner sep=3pt, font=\scriptsize},
  hand/.style={->, thick, antiholo},
  lab/.style={font=\tiny, align=center}]
\node[stage] (s1) at (-4.9,0) {\textbf{(a) Energy \& drive}\\[2pt]
  {\tiny $\Phi = E + \beta C$ \eqref{eq:phasor-energy}}\\[1pt]
  {\tiny $D_\ell = -2\,\partial\Phi/\partial\bar z_\ell$ \eqref{eq:drive}}};
\node[stage] (s2) at (-1.63,0) {\textbf{(b) Descent on the torus}\\[2pt]
  {\tiny $z \leftarrow (1-\dt)z + \dt\,\unitproject(D)$ \eqref{eq:discrete-settle}}\\[1pt]
  {\tiny fixed point $z^* = \unitproject(D)$}};
\node[stage] (s3) at (1.63,0) {\textbf{(c) Nudge \& probe}\\[2pt]
  {\tiny $\beta(t) = \varepsilon\cos\omega_p t$}\\[1pt]
  {\tiny $z(t) - z^\circ \approx$}\\[-1pt] {\tiny $\varepsilon\,(\dd z^*/\dd\beta)\cos\omega_p t$}};
\node[stage] (s4) at (4.9,0) {\textbf{(d) Lock-in}\\[-1pt] \textbf{demodulation}\\[2pt]
  {\tiny $\partial\mathcal{L}/\partial W_\ell =$}\\[-1pt] {\tiny $-\tfrac{2}{\varepsilon}\big\langle H_\ell\cos\omega_p t\big\rangle_T$ \eqref{eq:lockin-gradient}}};
\draw[hand] (s1) -- (s2) node[midway, above, lab, antiholo] {$D_\ell$};
\draw[hand] (s2) -- (s3) node[midway, above, lab, antiholo] {$z^\circ$};
\draw[hand] (s3) -- (s4) node[midway, above, lab, antiholo] {$H_\ell(t)$};
\node[lab, probeC, anchor=south] (in) at (-4.9,1.45) {input $x$, target $y$, weights $W$};
\draw[->, probeC] (in) -- (s1);
\node[lab, axisgray, anchor=south] at (1.63,1.45) {valid for adiabatic $\omega_p$ and small $\varepsilon$};
\node[lab, axisgray, anchor=south east] at (6.15,1.88) {every stage is proved in Supplementary Proof~A};
\draw[->, thick, hebbC, rounded corners=4pt] (s4.south) -- ++(0,-0.55) -| (s1.south)
  node[pos=0.25, below, lab, hebbC] {update $W_\ell \leftarrow W_\ell - \eta\,\partial\mathcal{L}/\partial W_\ell$; the new weights define the next energy};
\end{tikzpicture}
\caption{We recapitulate the stages of our derivation, and highlight the relevant formulas and quantities passed from each stage to the next. By defining an energy function on the network $\Phi$, we create an associated drive function $D$ which provides ``steps'' which can lower the energy of the network with inputs $x$, target outputs $y$, and weights $W$ (a). The per-layer drive $D_\ell$ is projected onto the unit torus in the complex plane, and a linear update step applies the drive to descend the energy function and ``land'' fixed network points $z^\circ$ on this torus (b). The contribution of the target outputs $y$ to the cost function, $\beta$, is varied following a cosine schedule with amplitude $\varepsilon$ and frequency $\omega_p$ (c). The resulting time-varying per-layer Hebbian energy, $H_\ell(t)$ can then be demodulated to provide an estimate of the loss function gradient w.r.t. network weights at each layer (d). These gradients can then be applied to step the network through an outer loop which reduces the loss function over time.}
\label{fig:roadmap}
\end{figure}
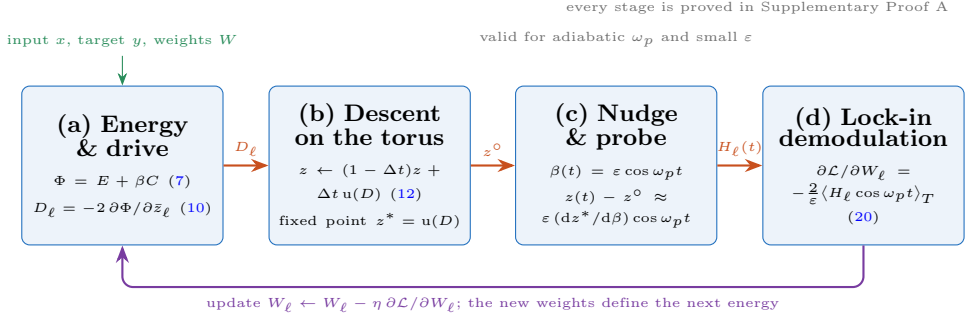

\subsubsection{Adiabatic Regime}
\label{sec:adiabatic_regime}

Our derivation of LIEP rests on several assumptions: one is that as the probe $\beta$ oscillates, the state $z$ still passes near the points $z_{\beta +}^*$ and $z_{\beta -}^*$ which would be reached by a ``static'' probe settling into each state individually. In other words LIEP still requires access to the same regions of the network's state as used by other probe methods; it is the methods of modulating $\beta$ and measuring the resulting gradient information which shift. As a result, both the frequency of the cosine probe $\omega_p$ and its amplitude $\varepsilon$ must be selected with respect to the rate at which the network converges to fixed points, $\kappa$. 

If our condition $\omega_p \ll \kappa$ is met --- that is, the relaxation period of the network is much shorter than the period of the probe --- we may be confident that the probe allows enough time for the network to sweep through the fixed points which would be extracted under a static, two-phase EP approach. However, this necessitates a long sampling time, and a very slow, quasi-DC signal may be difficult to demodulate. Conversely, if $\omega_p$ increases and approaches $\kappa$ --- the period of the probe approaches the relaxation time of the network --- the required path through the fixed states will not be followed, causing bias in any resulting gradient estimates. Additionally, the strength of the probe $\varepsilon$ must be chosen. Too small, and the probe will not produce usably measurable gradients. Too large, and non-linear terms of our estimate will begin to contribute and cause errors. 

As a result, LIEP must operate within an ``adiabatic'' basin: in this region, a gradual probe frequency allows the network sufficient (but not excessive) time and stimulation to relax to equilibrium points as it is influenced by the sinusoidally modulated cost function influencing $z(t)$ through time. We characterize this regime in the following results.

\section{Results}

To demonstrate the learning capabilities of LIEP, we utilize a simple network following the topology of a multi-layer perceptron (MLP) with 1 input layer of 784 oscillators, 1 hidden layer of 256 oscillators, and 1 output layer of 64 oscillators. Following the requirements of LIEP, each layer is symmetrically connected with 1 set of weights. Additionally, as the output of each oscillator is normalized into a phase, we may equivalently consider each oscillator combined with this non-linear topological constraint a ``neuron,'' and will employ this term as the base unit from which networks are constructed.

\subsection{Empirical Demonstrations}\label{sec:fashionmnist}

The FashionMNIST dataset is applied as a benchmark to validate the capability of LIEP to train a multi-layer network on a simple classification task \cite{xiaoFashionMNISTNovelImage2017}. Following the origins of our architecture in hyperdimensional (HD) computing, we define a codebook of pseudo-orthogonal phases which serve as ``labels'' for the 64 output neurons to align to. The dissimilarity of the output neurons to the correct class's code is set as the energy of the network's cost function (see Eqn. \ref{eq:phasor-cost}). For the network to correctly classify an input image encoded into phases, it must reach a fixed point at which the 64 output neurons of the network align to the correct code (e.g., given an image of a T-shirt, the code at the output phases match the codebook entry for ``T-shirt''). 

The LIEP parameters $\omega_p$, $\varepsilon$, and $T$ were swept to establish the adiabatic basin for the network and determine the optimal region for effective training. As shown in Figure \ref{fig:adiabatic_basin}, a slow probe frequency combined with a moderate probe amplitude allows the network to reach stable fixed points and to provide useful gradient signals. These gradients, applied through a standard stochastic gradient descent (SGD) update rule across the network, allow the loss function to smoothly decline to a minimum (Figure \ref{fig:adiabatic_basin}d). For these shallow, small networks, the accuracies of 83-84\% approach those reached by standard backpropagation (BP) applied to phase-based and conventional architectures (Figure \ref{fig:adiabatic_basin}c). However, as the LIEP parameters drift from the adiabatic zone, quality of gradients provided by LIEP compared to standard BP quickly degrade, leading to a high rate of divergent training runs which end in failure. As the probe frequency increases past the natural relaxation rate of the network $\kappa$, the information provided by the lock-in is no longer useful, and every run ends in failure (Figure \ref{fig:adiabatic_basin}b). As a result, $\omega_p$ can be set fairly freely to an arbitrary range of values as long as it does not approach the relaxation rate of the network itself.

\begin{figure}
    \centering
    \includegraphics[width=1\linewidth]{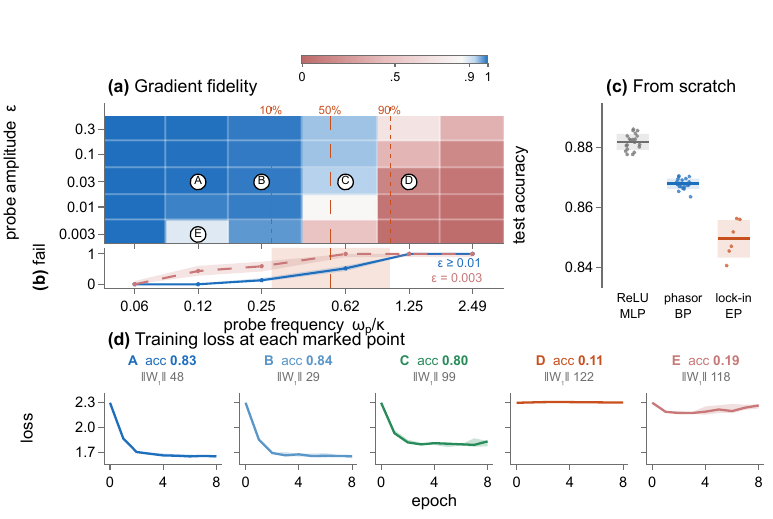}
    \caption{LIEP is controlled by 3 major parameters: the frequency ($\omega_p$) and amplitude ($\varepsilon$) of the cosine probe, and the number of integration periods ($T$) over which the lock-in signal is integrated. The quality of the gradients measured by their cosine similarity to BP-derived values is plotted across a grid where $\omega_p$ and $\varepsilon$ are swept, establishing the existence of an ``adiabatic'' basin. In this regime, the slow probe frequency and moderate probe amplitude allow the network to settle to steady states and produce high-quality gradient information (a). If the probe becomes too fast, approaching the rate at which the network descends to fixed points ($\kappa$), the energy dynamics of the system can no longer produce this information, and the ability of the networks to train successfully collapses, showing high failure rates (b) and non-convergence of the loss function (d). Networks trained \textit{ab-initio} with LIEP in the adiabatic region approach performance levels of standard BP on the same network architectures (c).}
    \label{fig:adiabatic_basin}
\end{figure}
These results provide a simple example of applying LIEP to train oscillatory neural networks ``ab-initio:'' the weights are initialized randomly. Next, we examine a different case which may be more applicable to our envisioned end-domain: tuning existing, trained networks to overcome the device variability which is expected in all manufactured hardware. 

\subsection{Impairment \& Recovery}\label{sec:recovery}

Analog computing generally depends on utilizing computing elements which are capable of many states to carry out a computation physically, rather than emulating it through digital logic. For instance, the matrix-vector-multiplication (MVM) which is ubiquitous within artificial neural networks may be computed in an analog fashion by encoding values in voltages or currents and utilizing crossbar arrays of resistors to physically scale (multiply) and combine (sum) the signals to realize the MVM \cite{xiaTechnologicalExplorationRRAM2016, raoThousandsConductanceLevels2023}. 

While this approach can be highly effective and has been scaled to large demonstrations, two major roadblocks have prevented it from reaching mainstream success. One is the high power and space requirements required to transcode digital/analog signals for computation, and the other is overcoming the variability inherent in the devices used in the computation \cite{mannocciInmemoryComputingEmerging2023, olin-ammentorpAdaptingSpikingNetworks2019}. Here, we focus our attention on the latter issue.

As a result of device-level variability, deploying digitally pre-trained neural networks onto an analog inference platform often requires one or several steps of conventional, BP-based corrections to lead to good performance \cite{wanComputeinmemoryChipBased2022}. While effective, this requires a lengthy and intensive process required for each deployment utilizing digital hardware platforms. Furthermore, this initial corrective tuning does not guarantee long-term performance; device level drifts and failures can commonly occur. In light of this, we propose that rather than viewing LIEP as a first-class algorithm for ab-initio training --- which is highly suited to the large-scale, digital systems massively deployed today --- it may be most effectively applied as an \textit{in-situ} training algorithm to allow analog platforms to self-recover performance rather than resorting to external measurement and correction techniques.

With this goal in mind, we carry out several experiments in which the weights of a network are distorted following a distribution of errors: Gaussian, log-normal scaling, and random shorts/opens (Figure \ref{fig:weight_correction}). Using shallow networks trained via standard BP (which may be converted into EP-compatible networks by symmetrizing their connectivity, with a bounded loss in accuracy --- see Supplementary Proof~B for details), the trained weights are distorted following one of these four patterns at a given level of severity. After a distortion is applied, several methods are used to attempt to recover the damaged network's performance: standard BP on the entire network, BP on only the readout layer, and LIEP. Having previously established the adiabatic zone for effective operation of LIEP, we now apply it with fixed parameters for this and the remaining experiments (Table \ref{tab:probeparams}). 
\begin{table}[tbp]
\centering
\footnotesize
\setlength{\tabcolsep}{3.4pt}
\caption{Lock-in probe settings behind each figure. $T$ is the demodulation
window in integrator steps, $n_\text{cyc}$ periods of the probe; a further
$T_\text{warm}$ periods are discarded before accumulation begins. All runs use
$784\to256\to64$ at initial weight scale $0.4$. With $\kappa=0.080$, the
$\omega_p=0.02$ used for the impairment studies is $\omega_p/\kappa=0.25$ ---
point~B of Fig.~\ref{fig:adiabatic_basin}, well inside the adiabatic basin and about a
factor of two below the fitted boundary. Two settings are \emph{not} shared
across Fig.~\ref{fig:weight_correction}: its panel (b) uses two warm-up periods and the
full $60$k training set, while the paired comparison in (c)--(d) uses one and
$10$k. The paired design is internally controlled, but accuracies are not
directly comparable between the two.}
\label{tab:probeparams}
\begin{tabular}{llccccccrr}
\toprule
Experiment & Fig. & $\varepsilon$ & $\omega_p$ & $n_\text{cyc}$ & $T_\text{warm}$
  & $T_\text{free}$ & $\mathrm{d}t$ & $T$ & Train \\
\midrule
Adiabatic grid        & \ref{fig:adiabatic_basin}a,b & \emph{swept} & \emph{swept} & \emph{swept} & 1 & 200 & 0.5 & \emph{varies} & --- \\
Basin training        & \ref{fig:adiabatic_basin}c,d & 0.003--0.03  & 0.01--0.10   & 4            & 1 & 200 & 0.5 & \emph{varies} & 60k \\
Impairment sweep      & \ref{fig:weight_correction}a--d & 0.03         & 0.02         & 4            & 2 & 200 & 0.5 & 2512 & 60k \\
Three-factor identity & \ref{fig:3factor_rule} & 0.05         & 0.05         & 4            & 2 & 100 & 0.1 & 5028 & --- \\
\bottomrule
\end{tabular}
\end{table}
\begin{figure}
    \centering
    \includegraphics[width=1.0\linewidth]{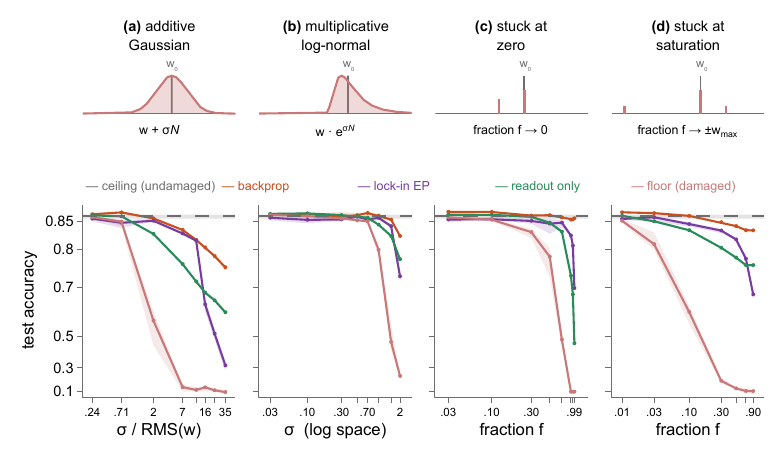}
    \caption{Rather than competing directly with standard backpropagation (BP) for large-scale training of digital networks, LIEP may be best applied as an analog-native training algorithm for correcting device-level distortions as pretrained networks are deployed onto devices. Many analog devices, such as resistive memories, have limited precision which causes the desired weights to be mapped to a variety of states. This mapping may follow a symmetric Gaussian curve (a), a log-normal scaling (b), or map weights to entirely broken ``open'' devices (c) or ``open/shorted'' devices (d). Each of these cases is applied at varying levels of severity, and BP on the entire network or only the readout layer is applied as a control correction. These methods are compared against the experimental LIEP. For many levels of disruption, LIEP can provide recovery approaching that of full-network BP. All accuracies plotted are the median of 8 samples.}
    \label{fig:weight_correction}
\end{figure}

The results of these experiments are plotted in Figure \ref{fig:weight_correction}. For each of the tested distortions, LIEP can effectively recover performance to a level similar to BP on the full network for moderate levels of impairment on all distortion types. As Gaussian noise is added to trained network weights, LIEP remains competitive with full-network BP for a wide range of distortions until the noise introduced by the distortion begins to overwhelm the trained weights, at which point the network essentially must be retrained rather than recovered (Fig. \ref{fig:weight_correction}a). A log-normal scaling, representing the asymmetric updates which are common in many resistive memories \cite{raoThousandsConductanceLevels2023, liuVortexVariationawareTraining2015, chenAcceleratorfriendlyNeuralnetworkTraining2017}, can also be successfully overcome using LIEP, as well as a significant fraction of device opens and shorts (Fig. \ref{fig:weight_correction}b-d). 

\subsection{3-Factor Learning Rule}\label{sec:three-factor}

Having supported the case that LIEP may be deployed on-device to correct networks which have moderate levels of analog distortions, we now return to the formulation of the method. As derived, LIEP is a ``top-down'' approach, where information is still collected across the entire network and broadcasted back to individual weights and parameters. To be compatible with neuromorphic and analog systems, a ``bottom-up'' three-factor formulation of LIEP which can be implemented in a per-element fashion is desirable \cite{fremauxNeuromodulatedSpikeTimingDependentPlasticity2016}. 

Thankfully, many of the quantities required to compute LIEP can be directly tracked at a per-neuron and per-synapse level, rather than via globally synchronized and broadcast values. To begin, we revisit the final update rule (Eqn. \ref{eq:lockin-gradient}) and examine the quantities it depends on: the Hebbian energy term, symmetric synaptic weights, and cosine probe (Fig. \ref{fig:3factor_rule}). The Hebbian energy is available locally to neurons, as it can be reformulated as the cosine of the difference in the phases between the pre- and post-synaptic neurons ($z_{\ell-1,j}(t)$ and $z_{\ell,i}(t)$, respectively) at a given point in time, reduced by the energy at the free point (Fig. \ref{fig:3factor_rule}b). This local energy, $e_{ij}$, is compared against the global probe applied at the output layer, $\beta(t)$ (Fig. \ref{fig:3factor_rule}c). This global signal allows the relative ``direction'' of the local energy term $e_{ij}$ to be demodulated and accumulated into a per-synapse eligibility trace $E_{ij}(t)$ \cite{bellecSolutionLearningDilemma2020}. Locally integrating and applying the resulting updates of this eligibility trace provides an equivalent three-factor formulation of LIEP which matches the original formulation to high numerical precision (Table \ref{tab:threefactor}). 

\begin{figure}
    \centering
    \includegraphics[width=1.0\linewidth]{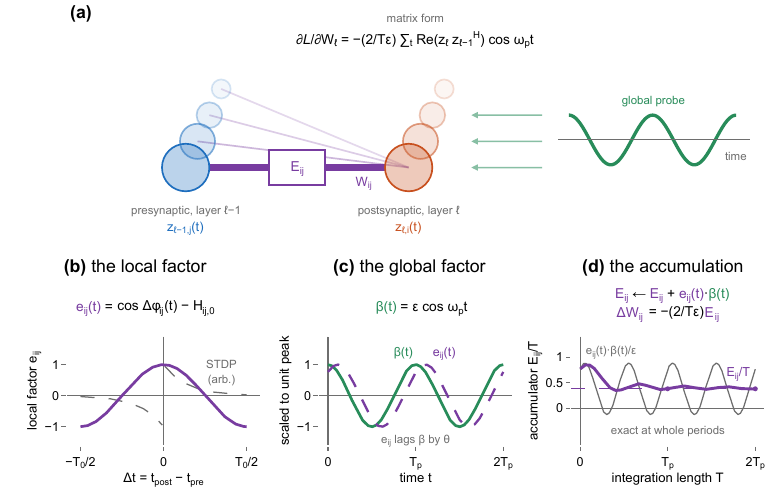}
    \caption{LIEP may be reformulated as a local, 3-factor learning rule. We begin by revisiting the final, matrix form of the loss-reducing weight update step (a). The quantities required for this update can be separated into locally-available energy information at each synapse (b) which measures the phase alignment of the pre- and post-synaptic neurons. This alignment is compared to the current state of the global probe (c), which allows the demodulated weight update to be accumulated and integrated into an eligibility trace (d) that can be used to update each weight locally.}
    \label{fig:3factor_rule}
\end{figure}

\begin{table}[tbp]
\centering
\footnotesize
\caption{The numerical results of the three-factor formulation of LIEP are compared against the original, network-level implementation. The same shallow MLP network used previously (Sec. \ref{sec:fashionmnist}) is optimized using the original LIEP formulation (Sec. \ref{sec:LIEP}) and its 3-factor version derived here. For all components in the network, all match to relative errors less than several parts per thousand.}
\label{tab:threefactor}
\begin{tabular}{@{}lrrrr@{}}
\toprule
& & \multicolumn{1}{c}{Worst} & \multicolumn{2}{c}{Relative error} \\
\cmidrule(l){4-5}
Tensor & Components & $1-\cos$ & median & max \\
\midrule
$W_1$ & 200\,704 & $4\times10^{-6}$ & $1.2\times10^{-3}$ & $2.8\times10^{-3}$ \\
$W_2$ & 16\,384 & $3\times10^{-7}$ & $3.8\times10^{-4}$ & $5.4\times10^{-4}$ \\
\bottomrule
\end{tabular}
\end{table}

\section{Discussion}\label{sec12}

LIEP meets a concrete research need from the oscillatory computing community as well as a neuromorphic training approach which can be implemented and executed locally in neuromorphic and analog computing circuits \cite{todri-sanialComputingOscillatorsTheoretical2024}. Additionally, in contrast to many EP methods which require manually settling to different ``nudged'' points, LIEP can inject a time-varying probe continuously after a free warm-up period, rather than settling to different fixed points discretely \cite{laborieuxScalingEquilibriumPropagation2021}. However, while demonstrating the basic proof-of-concept of LIEP, our extant results presented here leave many points open for further consideration and development. 

In part, LIEP was developed as a complement to our previous training algorithm, the $\varphi$-SSM \cite{olin-ammentorpPhaseStateSpace2026}. The $\varphi$-SSM method is a digitally native approach in which large systems of oscillators may be trained on massively parallel supercomputing hardware incorporating GPUs to compute update to the system at arbitrary points in time via parallel scan. The massive parallelization of the digital training approach can cover the exploration and simulation of physical, analog systems, producing weights and architectures which can be deployed onto real, analog chips. Once deployed LIEP may be used as the ``fine-tuning'' algorithm to correct for device non-ideality, drift, and failures. However, in this work, we have thus far only demonstrated that shallow networks trained via $\varphi$-SSM and operating over a single input converge to the outputs expected from a symmetric, EP-compatible network. More work will be required to establish further equivalences in deeper networks, and investigate whether multi-step, temporal input and output sequences natively handled by $\varphi$-SSMs can be used with the LIEP algorithm. In other words, while we have established that both methods can provide training to reach the same fixed points, $\varphi$-SSMs can train stable orbits, and this has not yet been shown for LIEP. 

Further consideration is required on this point, but we suggest that following the hierarchy of frequency bands and diversity of neurons present in the natural brain, a multi-band carrier signal and networks of neurons with multiple resonant frequencies may be used to construct networks which can natively recreate stable orbits, rather than points, by the combination of multiple fixed points in a Fourier space interfering to produce a shared, temporal signal --- an ``epicyclic'' network, which could utilize the same primitives as LIEP but extend them to a more complex end-use case. 

Furthermore, while the $\varphi$-SSM can be scaled to large networks of thousands or millions of neurons which utilize popular residual and attentional based network architectures to avoid gradient collapse and allow for tractability on large, temporal datasets, four barriers remain to extend LIEP past the small networks we have demonstrated here. 

One is computational: rather than utilizing several discrete steps which are used to calculate an update, LIEP uses \textit{thousands} of simulated steps to sweep through the cosine probe injected at every step and converge to the desired state. In a physical, analog computing system, these steps are calculated intrinsically by the natural energetic dynamics of the system as it descends to its minima; in a digital computer, each step must be calculated and emulated, slowing it down compared to traditional digital approaches. As a result, large-scale \textit{in-silico} simulations of LIEP may be better emulated rather than computed exactly. For instance, reverting to the static, three-point method provides the states needed to estimate LIEP's may be sufficient to simulate it effectively in a digital system --- however, concrete and experimental proofs must be made to show if these methods closely track each other, and what conditions may be necessary to ensure a good match.

Secondly, scaling LIEP past small networks requires extending the methods we have demonstrated previously to avoid gradient collapse in oscillator-based networks to be compatible with EP-based methods. Currently, the shallow phase-based networks utilized here for classification are constructed utilizing the ``bundling'' and ``similarity'' operations of HD computing \cite{kleykoSurveyHyperdimensionalComputing2021, thomasTheoreticalPerspectiveHyperdimensional2021}. From an HD perspective, our derivation established energy functions based on these two operations which are varied to extract the loss gradient. The ``missing ingredient'' to complete the space of computable HD functions is ``binding,'' which is utilized in more advanced and deep architectures to overcome gradient collapse in phase-based networks. We propose that all HD computing terms may be defined in terms of an equivalent energy function, extending the oscillatory computing paradigm to include an oscillatory gradient descent method; this may allow for deep, large, and complex networks and algorithms to be expressed and trained on oscillatory computing hardware. 

Thirdly, the update rule we have derived does remain dependent on a per-neuron reference to a global probe signal, which is generally only applied at the output layers. Physically realizing the distribution of this global signal is costly, and a fully local realization of LIEP would be preferable. However, local ``artifacts'' of the global probe exist in each layer, and by modifying the probe itself and/or the decoding of this local information, the 3-factor formulation may potentially be realized as a fully local rule. 

Lastly, the energetic costs of implementing and executing LIEP on physical, oscillatory hardware should be investigated and characterized. A large spectrum of oscillatory hardware exists, with resonant frequencies from Hz to THz \cite{csabaCoupledOscillatorsComputing2020}. The energetic cost of driving these systems, ability to construct and connect them on large scales, and methods for driving these systems will have a large impact on the ability to practically construct systems which can physically realize and execute LIEP.

\section{Conclusion}\label{sec13}

In this work, we derived and demonstrated a novel method for training phase-based oscillatory networks via a variant of equilibrium propagation. This ``lock-in'' equilibrium propagation method (LIEP) was demonstrated at effectively training networks on a simple but non-trivial classification task, as well as being highly effective to recover a pre-trained system which had weights perturbed by a number of non-idealities corresponding to defects in real emerging memory devices. Additionally, we provided a local, three-factor formulation of LIEP which can allow for the same updates to be made while only referencing one global ``probe'' signal. We believe that the proof-of-concept demonstrations made here open the door to many further variants in the LIEP algorithm itself, as well as underlying networks, algorithms, and hardware implementations. 

\section{Methods}\label{sec11}

All algorithms were implemented in the Julia package, PhasorNetworks.jl. Simulations were run locally on an NVIDIA DGX Spark system, as well as the ALCF Polaris and Aurora systems. Proofs and codes were developed and checked by Claude Models (Fable 4.8-5.0) and Nemotron-3-Ultra. Code is available online at \href{https://github.com/wilkieolin/PhasorNetworks.jl}{https://github.com/wilkieolin/PhasorNetworks.jl}

\backmatter

\bmhead{Supplementary information}

Supplementary Proofs A (Full Derivation of Lock-in Equilibrium Propagation) and B (Backprop-to-Equilibrium-Propagation Transfer Guarantees for Phasor Networks) are included as appendices at the end of this document.

\bmhead{Acknowledgements}

The authors would like to thank Valerie Taylor as well as the ALCF staff for their support. Furthermore, we thank Michel Schanen for his development and support of the OneAPI.jl package. 

This work was supported by DOE ASCR BIA: A Co-Design Methodology to Transform Materials and Computer Architecture Research for Energy Efficiency. This material is based upon work supported by the U.S. Department of Energy, Office of Science, under contract number DE-AC02-06CH11357.

The submitted manuscript has been created by UChicago Argonne, LLC, Operator of Argonne National Laboratory (``Argonne''). Argonne, a U.S. Department of Energy Office of Science laboratory, is operated under Contract No. DE-AC02-06CH11357. The U.S. Government retains for itself, and others acting on its behalf, a paid-up nonexclusive, irrevocable worldwide license in said article to reproduce, prepare derivative works, distribute copies to the public, and perform publicly and display publicly, by or on behalf of the Government. The Department of Energy will provide public access to these results of federally sponsored research in accordance with the DOE Public Access Plan. \href{http://energy.gov/downloads/doe-public-access-plan}{http://energy.gov/downloads/doe-public-access-plan}

AI was applied in this work to generate supplemental proofs, code, and figures, as well as to collect relevant literature and review drafts for typos and other grammatical errors. Exclusive of the supplementary proofs, no AI-generated text is included in this manuscript.

\appendix
\renewcommand{\thesubsection}{A.\arabic{subsection}}
\renewcommand{\thesubsubsection}{A.\arabic{subsection}.\arabic{subsubsection}}
\section*{Supplementary Proof A}
\setcounter{theorem}{0}\setcounter{remark}{0}\setcounter{definition}{0}\setcounter{equation}{0}\setcounter{table}{0}\setcounter{figure}{0}
\renewcommand{\thesection}{A}
\renewcommand{\theequation}{A.\arabic{equation}}\renewcommand{\theHequation}{A.\arabic{equation}}
\renewcommand{\thetheorem}{A.\arabic{theorem}}\renewcommand{\theHtheorem}{A.\arabic{theorem}}
\renewcommand{\theremark}{A.\arabic{remark}}\renewcommand{\theHremark}{A.\arabic{remark}}
\renewcommand{\thedefinition}{A.\arabic{definition}}\renewcommand{\theHdefinition}{A.\arabic{definition}}
\renewcommand{\thetable}{A.\arabic{table}}\renewcommand{\theHtable}{A.\arabic{table}}
\renewcommand{\thefigure}{A.\arabic{figure}}\renewcommand{\theHfigure}{A.\arabic{figure}}
%
%

\section{Full Derivation of Lock-in Equilibrium Propagation}\label{sec:liep_proof}

This appendix supplies the details omitted from Section~\ref{sec:LIEP}, in the same order: the state space, energy and cost (\ref{app:setting}); the Wirtinger gradients that define the drive on each oscillator (\ref{app:wirtinger}); the constrained gradient flow on the torus and its convergence (\ref{app:flow}); the discrete settling update and its fixed points (\ref{app:discrete}); the fixed-point map and the passage from two-phase nudging to a continuous probe (\ref{app:nudged}); the loss gradient and its extraction by lock-in demodulation (\ref{app:lockin}); and the hypotheses that define the operating regime (\ref{app:regime}). Two further results have no counterpart in the main text: the non-holomorphic dependence of a torus-valued equilibrium on a complex nudge, which rules out the contour estimator of holomorphic equilibrium propagation (\ref{app:nonholo}), and the change of frame to rotating oscillators (\ref{app:frame}). Throughout, $E$ denotes a true energy, minimised when oscillators are aligned with their drives, and the nudged energy is $\Phi = E + \beta C$, as in standard equilibrium propagation \cite{scellierEquilibriumPropagationBridging2017}.

\subsection{Setting and Notation}\label{app:setting}

A phasor network has $L$ layers of states $z_\ell \in \T^{d_\ell} \subset \C^{d_\ell}$ on the unit torus
\begin{equation}\label{eq:torus-app}
\T^{n} = \{ z \in \C^{n} : |z_j| = 1 \text{ for all } j \},
\end{equation}
real weight matrices $W_\ell \in \R^{d_\ell \times d_{\ell-1}}$, a fixed input $z_0 = x$, and a target $y \in \T^{d}$ with $d = d_L$ \cite{olin-ammentorpDeepPhasorNetworks2022,olin-ammentorpHyperdimensionalComputingProvides2023}. We use the Hermitian inner product $\langle v, w\rangle = v^H w$, conjugate-linear in its first argument, and write $(\cdot)^H$ for the conjugate transpose. For unit vectors, $\Real\langle v, w\rangle = \sum_j \cos(\theta_{v,j} - \theta_{w,j})$ is the cosine similarity of the two phase patterns.

Subscript $0$ always denotes the input layer. The free equilibrium --- the equilibrium reached at $\beta = 0$ --- is written $z^\circ$, with layers $z^\circ_\ell$ and entries $z^\circ_{\ell,j}$. Coefficients of expansions in the nudge are written $u_1, u_2, u_3$, so that subscripts on $z$ are always layer indices and subscripts on $u$ are always orders in $\beta$. 

The energy of the network is the negative similarity between each layer and its weighted drive from the layer below,
\begin{equation}\label{eq:energy-app}
E(z) = -\sum_{\ell=1}^{L} \Real\langle W_\ell z_{\ell-1}, z_\ell \rangle
= -\sum_{\ell=1}^{L} \tfrac{1}{2}\big( z_\ell^H W_\ell z_{\ell-1} + z_{\ell-1}^H W_\ell^\top z_\ell \big),
\end{equation}
where $W_\ell^H = W_\ell^\top$ because the weights are real. Equation~\eqref{eq:energy-app} is the sum over layers of the per-layer term in Eq.~\eqref{eq:phasor_layer_energy} with the sign chosen so that $E$ is minimised at alignment. The output cost is one minus the similarity to the target,
\begin{equation}\label{eq:cost-app}
C(z_L, y) = 1 - \frac{1}{d}\Real\langle z_L, y \rangle
= 1 - \frac{1}{d}\sum_{j=1}^{d} \cos(\theta_{L,j} - \theta_{y,j}) \in [0, 2],
\end{equation}
as in Eq.~\eqref{eq:phasor-cost}, and the nudged energy and the loss are
\begin{equation}\label{eq:total-energy-app}
\Phi(z; \beta) = E(z) + \beta\, C(z_L, y),
\qquad
\mathcal{L}(W) = C\big(z^\circ_L(W), y\big),
\end{equation}
where $z^\circ(W)$ is the free equilibrium and $z^\circ_L$ its output layer. 
The loss is the cost of the free state; the nudge $\beta$ is a probe used only to extract its gradient.

\begin{remark}[Real weights and the Hermitian form]
The energy \eqref{eq:energy-app} is invariant under a global phase rotation $z_\ell \mapsto e^{\im\varphi} z_\ell$ applied to every layer, the input and the target: $\Real\langle W e^{\im\varphi} v, e^{\im\varphi} w\rangle = \Real\langle W v, w\rangle$ because $W$ is real. Complex weights would break this invariance and make the energy depend on absolute rather than relative phases. The bilinear form $v^\top w$ is holomorphic in $z$ but is not invariant either; the Hermitian form is the one that measures phase alignment.
\end{remark}

\begin{remark}[Relation to XY and Kuramoto models]\label{rem:xy} 
In angle coordinates $z_{\ell,j} = e^{\im\theta_{\ell,j}}$ the energy \eqref{eq:energy-app} reads
\[
E(\theta) = -\sum_{\ell=1}^{L}\sum_{i,j} (W_\ell)_{ij}\,\cos\big(\theta_{\ell,i} - \theta_{\ell-1,j}\big),
\]
a layered XY model with signed coupling matrix $W_\ell$ and no intra-layer coupling; the constrained flow of \ref{app:flow} is its Kuramoto-type gradient dynamics with the carrier of \ref{app:frame} supplying the common natural frequency. This form is also the more convenient one for machine-checked verification, since it is a smooth function on $\R^{n}$ and dispenses with the torus constraint and the tangent projector.
\end{remark}

\begin{remark}[Other costs]
Any cost that is real-differentiable in $z_L$ can replace \eqref{eq:cost-app}; the derivation below uses only its Wirtinger derivative. For classification against a codebook of phase patterns $c_1,\dots,c_K \in \T^{d}$, the cross-entropy $C = -\log \mathrm{softmax}_k\big(\tfrac1d\Real\langle c_k, z_L\rangle\big)_{y}$ is convenient, and its nudge is a softmax-weighted combination of the codewords.
\end{remark}

\subsection{Wirtinger Derivatives and the Drive}\label{app:wirtinger}

For a real-valued function $f$ of $z = x + \im y \in \C$ the Wirtinger derivatives are $\partial f/\partial z = \tfrac12(\partial_x - \im\,\partial_y) f$ and $\partial f/\partial \bar z = \tfrac12(\partial_x + \im\,\partial_y) f$. Identifying $\C \cong \R^2$, the gradient of $f$ is the complex number $\nabla f = \partial_x f + \im\,\partial_y f = 2\,\partial f/\partial \bar z$, so the direction of steepest descent is $-2\,\partial f/\partial\bar z$. Treating $z$ and $\bar z$ as independent, $\partial(v^H w)/\partial \bar w = 0$ and $\partial(v^H w)/\partial \bar v = w$, hence for the layer term
\begin{equation}
\frac{\partial}{\partial \bar z_\ell}\Big(-\Real\langle W_\ell z_{\ell-1}, z_\ell\rangle\Big) = -\tfrac12 W_\ell z_{\ell-1},
\qquad
\frac{\partial}{\partial \bar z_\ell}\Big(-\Real\langle W_{\ell+1} z_{\ell}, z_{\ell+1}\rangle\Big) = -\tfrac12 W_{\ell+1}^\top z_{\ell+1},
\end{equation}
and for the cost $\partial C/\partial \bar z_L = -\tfrac{1}{2d}\,y$. Each hidden state appears in two coupling terms, one from below and one from above, so
\begin{align}
\frac{\partial \Phi}{\partial \bar z_\ell} &= -\tfrac12\big( W_\ell z_{\ell-1} + W_{\ell+1}^\top z_{\ell+1} \big), \qquad \ell = 1,\dots,L-1, \label{eq:grad-hidden-app}\\
\frac{\partial \Phi}{\partial \bar z_L} &= -\tfrac12\big( W_L z_{L-1} + \tfrac{\beta}{d}\, y \big), \label{eq:grad-output-app}
\end{align}
which are Eqs.~\eqref{eq:grad-hidden}--\eqref{eq:grad-output}. The factor $\tfrac12$ is the usual Wirtinger factor and disappears in the real gradient. We call the negative real gradient the \emph{drive} on layer $\ell$,
\begin{equation}\label{eq:drive-app}
D_\ell := -2\,\frac{\partial \Phi}{\partial \bar z_\ell} = W_\ell z_{\ell-1} + W_{\ell+1}^\top z_{\ell+1},
\qquad W_{L+1}^\top z_{L+1} \equiv \tfrac{\beta}{d}\, y,
\end{equation}
the last identity being a notational convention for the output layer: the target acts on layer $L$ like a fictitious layer above it whose feedback is the nudge. This is Eq.~\eqref{eq:drive} of the main text.

\begin{remark}[Scale of the drive]
Because the unit projection used below satisfies $\unitproject(cD) = \unitproject(D)$ for every $c > 0$, an overall positive factor in $D_\ell$ changes neither the equilibria nor the gradient estimate. Only the relative weight of the coupling and nudge terms matters, and both carry the same factor in \eqref{eq:drive-app}.
\end{remark}

\subsection{Constrained Gradient Flow and Convergence}\label{app:flow}

The ambient flow of Eq.~\eqref{eq:continuous-flow}, $\dot z_\ell = -\gamma z_\ell + D_\ell$, does not preserve $|z_{\ell,j}| = 1$. The tangent space of the torus at $z \in \T^n$ and the orthogonal projector onto it are
\begin{equation}\label{eq:tangent-app}
T_z\T^n = \{ v \in \C^n : \Real(\bar z_j v_j) = 0 \ \forall j \},
\qquad
\Pi_z(v) = v - \Real(\bar z \odot v)\odot z,
\end{equation}
with $\odot$ the entrywise product: tangent vectors are purely rotational. Since $\Pi_z(z) = 0$, the decay term is normal to the torus and is removed by the projection, $\Pi_z(-\gamma z + D) = \Pi_z(D)$. The constrained dynamics is therefore the Riemannian gradient flow of $\Phi$ on the torus,
\begin{equation}\label{eq:riemannian-flow-app}
\dot z_\ell = \Pi_{z_\ell}(D_\ell) = -\Pi_{z_\ell}\big(\nabla_{z_\ell}\Phi\big), \qquad \ell = 1,\dots,L .
\end{equation}

\begin{proposition}[$\Phi$ is a Lyapunov function]\label{prop:lyapunov}
Along \eqref{eq:riemannian-flow-app},
$\dd\Phi/\dd t = -\sum_\ell \|\Pi_{z_\ell}(D_\ell)\|^2 \le 0$, with equality exactly at the critical points of $\Phi$ on the torus, $\Pi_{z_\ell}(D_\ell) = 0$ for all $\ell$. Since $\Phi$ is continuous on the compact torus, every trajectory converges to the set of critical points.
\end{proposition}
\begin{proof}
With the real inner product $\Real\langle v, w\rangle$ on $\C^n \cong \R^{2n}$, the chain rule and $\nabla_{z_\ell}\Phi = -D_\ell$ give $\dd\Phi/\dd t = \sum_\ell \Real\langle \nabla_{z_\ell}\Phi, \dot z_\ell\rangle = -\sum_\ell \Real\langle D_\ell, \Pi_{z_\ell}(D_\ell)\rangle = -\sum_\ell \|\Pi_{z_\ell}(D_\ell)\|^2$, because $\Pi_z$ is an orthogonal projector and $\Pi_z^2 = \Pi_z$. $\Phi$ is bounded below on the torus and non-increasing, and the invariance principle of LaSalle yields convergence to the largest invariant subset of $\{\dd\Phi/\dd t = 0\}$, which is the critical set.
\end{proof}

\begin{remark}[Convergence to a single point]\label{rem:lojasiewicz} 
Proposition~\ref{prop:lyapunov} gives convergence to the critical \emph{set}, which for a general Lyapunov function need not imply convergence to a single point. Here it does: in the angle coordinates of Remark~\ref{rem:xy}, $\Phi$ is a trigonometric polynomial and hence real-analytic, so the {\L}ojasiewicz gradient inequality applies and every trajectory of the constrained flow has finite length and converges to one critical point. This is the standard argument for analytic gradient flows and is what licenses speaking of \emph{the} free equilibrium $z^\circ$ reached from a given initialisation.
\end{remark}

\begin{remark}[Symmetric feedback]
The drive \eqref{eq:drive-app} is a gradient only because the feedback into layer $\ell$ uses the transpose $W_{\ell+1}^\top$ of the weights that carry its output upward. With an independent feedback matrix the flow is no longer a gradient flow, $\Phi$ is no longer a Lyapunov function, and the equilibrium-propagation theorem of \ref{app:lockin} does not apply. This is the same weight-symmetry requirement as in standard equilibrium propagation \cite{scellierEquilibriumPropagationBridging2017}.
\end{remark}

At a critical point each entry of the drive is real relative to its state, $\Pi_{z}(D) = 0 \iff D_j \in \R\, z_j$, so for $D_j \neq 0$ either $z_j = \unitproject(D_j)$ (aligned) or $z_j = -\unitproject(D_j)$ (anti-aligned). Holding all other entries fixed and writing $D_j = r_j z_j$, the energy varies with the phase $\theta_j$ of $z_j$ as $-|D_j|\cos(\theta_j - \theta_{D_j})$, which is minimised at $\theta_j = \theta_{D_j}$ and maximised at $\theta_j = \theta_{D_j} + \pi$: an anti-aligned entry ($r_j < 0$) sits at the maximum and is therefore a strict local maximum in that coordinate, so the critical point is not a local minimum of $\Phi|_{\T}$ and, the flow being a gradient flow, is not Lyapunov stable. Stable equilibria of \eqref{eq:riemannian-flow-app} therefore have every oscillator aligned with its drive,
\begin{equation}\label{eq:fixed-point-app}
z_{\ell}^* = \unitproject\big(D_\ell(z^*)\big) = \frac{D_\ell(z^*)}{|D_\ell(z^*)|} \quad\text{entrywise}, \qquad \ell = 1,\dots,L .
\end{equation}

\subsection{Discrete Settling Update}\label{app:discrete}

In practice the equilibrium is reached by iterating Eq.~\eqref{eq:discrete-settle}, with all layers updated from the previous iterate,
\begin{equation}\label{eq:settle-app}
z_\ell^{(k+1)} = (1-\dt)\, z_\ell^{(k)} + \dt\; \unitproject\big(D_\ell(z^{(k)})\big),
\qquad 0 < \dt \le 1,
\end{equation}
where $\unitproject(g) = g/|g|$ acts entrywise.

\begin{proposition}[Fixed points of the settling update]\label{prop:fixed-points}
Let $z$ be such that $D_{\ell,j}(z) \neq 0$ for all $\ell, j$. Then $z$ is a fixed point of \eqref{eq:settle-app} if and only if $z_\ell = \unitproject(D_\ell(z))$ for all $\ell$, that is, if and only if $z$ is an aligned critical point \eqref{eq:fixed-point-app} of $\Phi$ on the torus. Moreover the closed unit polydisc is forward invariant: if $|z_{\ell,j}^{(k)}| \le 1$ for all $\ell,j$ then $|z_{\ell,j}^{(k+1)}| \le 1$.
\end{proposition}
\begin{proof}
A fixed point satisfies $z_\ell = (1-\dt) z_\ell + \dt\,\unitproject(D_\ell)$, i.e.\ $z_\ell = \unitproject(D_\ell)$ since $\dt > 0$; the converse is immediate. Such a $z$ has unit modulus and each entry is a positive multiple of its drive, which is \eqref{eq:fixed-point-app}. For the invariance, the triangle inequality with $|\unitproject(\cdot)| = 1$ gives $|z^{(k+1)}_{\ell,j}| \le (1-\dt)|z^{(k)}_{\ell,j}| + \dt \le (1-\dt) + \dt = 1$.
\end{proof}

\begin{proposition}[$\Phi$ descends along the relaxation flow]\label{prop:relax-lyapunov}
Let
\begin{equation}\label{eq:relax-flow-app}
\dot z_\ell = -z_\ell + \unitproject\big(D_\ell(z)\big), \qquad \ell = 1,\dots,L,
\end{equation}
be the flow of which \eqref{eq:settle-app} is the forward Euler step, and suppose $D_{\ell,j}(z) \neq 0$ along the trajectory. The closed unit polydisc $\{|z_{\ell,j}| \le 1\}$ is forward invariant for \eqref{eq:relax-flow-app}, and on it
\begin{equation}\label{eq:relax-descent-app}
\frac{\dd\Phi}{\dd t}
= \sum_{\ell,j}\Big[\Real\big(\bar D_{\ell,j}\,z_{\ell,j}\big) - |D_{\ell,j}|\Big]
= -\sum_{\ell,j}|D_{\ell,j}|\,\big(1 - |z_{\ell,j}|\cos\alpha_{\ell,j}\big) \;\le\; 0,
\end{equation}
where $\alpha_{\ell,j}$ is the angle between $z_{\ell,j}$ and $D_{\ell,j}$, with equality exactly when $z_{\ell,j} = \unitproject(D_{\ell,j})$ for every oscillator. Hence $\Phi$ is a Lyapunov function for \eqref{eq:relax-flow-app} on the whole polydisc --- off the torus as well as on it --- and every trajectory converges to the set of aligned fixed points \eqref{eq:fixed-point-app}.
\end{proposition}
\begin{proof}
The Wirtinger computation of \ref{app:wirtinger} nowhere uses $|z_{\ell,j}| = 1$, so $\nabla_{z_\ell}\Phi = -D_\ell$ holds on all of $\C^{n}$. With the real inner product and $\bar D\,\unitproject(D) = |D|$,
\[
\frac{\dd\Phi}{\dd t}
= \sum_{\ell,j}\Real\Big(\overline{(-D_{\ell,j})}\big(-z_{\ell,j} + \unitproject(D_{\ell,j})\big)\Big)
= \sum_{\ell,j}\Big[\Real\big(\bar D_{\ell,j} z_{\ell,j}\big) - |D_{\ell,j}|\Big],
\]
and $\Real(\bar D_{\ell,j} z_{\ell,j}) = |D_{\ell,j}|\,|z_{\ell,j}|\cos\alpha_{\ell,j}$ gives the second form of \eqref{eq:relax-descent-app}. It is non-positive because $|z_{\ell,j}|\cos\alpha_{\ell,j} \le |z_{\ell,j}| \le 1$, with equality iff $|z_{\ell,j}| = 1$ and $\alpha_{\ell,j} = 0$, i.e.\ $z_{\ell,j} = \unitproject(D_{\ell,j})$. Forward invariance follows from $\dd|z_{\ell,j}|/\dd t = -|z_{\ell,j}| + \cos\alpha_{\ell,j} \le 0$ whenever $|z_{\ell,j}| = 1$. Finally, $\dd\Phi/\dd t = 0$ forces $\dot z = 0$, so the set on which $\Phi$ is stationary consists of equilibria and is trivially invariant; $\Phi$ is continuous on the compact polydisc, and LaSalle's invariance principle gives convergence to that set.
\end{proof}

\begin{remark}[Behaviour of the moduli]\label{rem:modulus} 
It is tempting to read a geometric contraction of the modulus deficit off the same triangle inequality, but that inequality bounds $|z^{(k+1)}|$ from \emph{above} and therefore yields $1 - |z^{(k+1)}_{\ell,j}| \ge (1-\dt)\big(1 - |z^{(k)}_{\ell,j}|\big)$, the opposite of a contraction. No contraction holds in general. Writing $\alpha_k$ for the angle between $z^{(k)}_{\ell,j}$ and $\unitproject\big(D_{\ell,j}(z^{(k)})\big)$, the exact identity is
\begin{equation}\label{eq:modulus-identity-app}
\big|z^{(k+1)}_{\ell,j}\big|^{2}
= (1-\dt)^{2}\big|z^{(k)}_{\ell,j}\big|^{2} + \dt^{2}
+ 2\dt(1-\dt)\,\big|z^{(k)}_{\ell,j}\big|\cos\alpha_k ,
\end{equation}
so the modulus decreases whenever the state and its drive are strongly misaligned: with $\dt = \tfrac12$, $z^{(k)}_{\ell,j} = 1$ and $\unitproject(D_{\ell,j}) = -1$ one gets $z^{(k+1)}_{\ell,j} = 0$, and the deficit jumps from $0$ to $1$. Setting $\alpha_k = 0$ in \eqref{eq:modulus-identity-app} recovers the exact relation $1 - |z^{(k+1)}_{\ell,j}| = (1-\dt)\big(1 - |z^{(k)}_{\ell,j}|\big)$, and by continuity the deficit contracts at rate $1-\dt + O(\alpha_k^2)$ once the phases have settled. The correct statement is therefore asymptotic: during the phase transient only forward invariance is guaranteed, while near an aligned equilibrium the moduli return to one geometrically at the fixed rate $1-\dt$, independently of the drive. This is the regime in which the Hebbian is read, and the fixed-point condition of Proposition~\ref{prop:fixed-points} does not involve the moduli at all.
\end{remark}

\begin{remark}[Initialisation]\label{rem:init} 
A zero initialisation is not admissible for \eqref{eq:settle-app}: with $z^{(0)} = 0$ the drive vanishes on every layer except the first, where $D_1 = W_1 x$, and on the output layer, where $D_L = (\beta/d)\,y$ vanishes as well at $\beta = 0$. The projector is undefined at $g = 0$, and any convention for $\unitproject(0)$ installs a preferred phase and breaks the rotation invariance of the energy (Remark~\ref{rem:soft}). Layers are therefore initialised on the torus --- at random phases, or by a single feedforward sweep $z_\ell \leftarrow \unitproject(W_\ell z_{\ell-1})$ from the input --- or the soft projection \eqref{eq:soft-project-app} is used, for which $\unitproject_\epsilon(0) = 0$ is continuous.
\end{remark}

\begin{remark}[Relation to projected Euler integration]\label{rem:integrators}
Equation~\eqref{eq:settle-app} is the forward Euler step of $\dot z_\ell = -z_\ell + \unitproject(D_\ell)$, a relaxation of every oscillator toward the direction of its drive. It is \emph{not} a discretisation of the constrained flow \eqref{eq:riemannian-flow-app}: its transient is a convex combination of the previous state and a unit vector and leaves the torus, whereas the constrained flow stays on it. Leaving the torus costs nothing, however: both flows descend $\Phi$, the constrained one on the torus (Proposition~\ref{prop:lyapunov}) and the relaxation one on the whole polydisc (Proposition~\ref{prop:relax-lyapunov}). The textbook projected Euler step, $z_\ell \leftarrow \unitproject\big(z_\ell + \dt\,\Pi_{z_\ell}(D_\ell)\big)$, is a consistent first-order integrator of \eqref{eq:riemannian-flow-app}. Its fixed points are all the critical points of $\Phi|_{\T}$, aligned and anti-aligned alike, since it is stationary wherever $\Pi_z(D) = 0$; the relaxation update \eqref{eq:settle-app} excludes the anti-aligned ones by construction, because $z = \unitproject(D)$ forces alignment. 
The anti-aligned critical points are unstable for both schemes, and at an aligned critical point the two notions of stability agree: in the angle coordinates of Remark~\ref{rem:xy} the Hessian of $\Phi|_{\T}$ factorises as $\nabla^2\Phi = \Lambda\,(I - J)$, where $\Lambda = \mathrm{diag}\big(|D_{\ell,j}|\big) \succ 0$ under (H2) of \ref{app:regime} and $J$ is the Jacobian of the angle map $\theta \mapsto \arg D(\theta)$, while the relaxation update linearises in the tangential directions to $(1-\dt)I + \dt J = I - \dt\,\Lambda^{-1}\nabla^2\Phi$. The eigenvalues of $\Lambda^{-1}\nabla^2\Phi$ are real and carry the signs of those of $\nabla^2\Phi$, so for $\dt$ small enough the relaxation update is locally asymptotically stable exactly at the local minima of $\Phi|_{\T}$ --- exactly the stable equilibria of the projected scheme. The two \emph{stable} equilibrium sets therefore coincide. Equilibrium propagation uses only the fixed-point map $z^*(\beta, W)$ restricted to stable equilibria (\ref{app:lockin}), so both schemes give identical free and nudged equilibria and identical gradient estimates. The relaxation form is used because it needs no tangent projection and, near an equilibrium, its modulus deficit contracts at the fixed rate $1-\dt$ independently of the drive (Remark~\ref{rem:modulus}).
\end{remark}

\begin{remark}[What is and is not proved about convergence]\label{rem:convergence-caveat} 
Propositions~\ref{prop:lyapunov} and \ref{prop:relax-lyapunov} are statements about continuous-time flows; what is iterated in practice is the forward Euler step \eqref{eq:settle-app}. The gap between them is the discretisation, not the departure from the torus: by Proposition~\ref{prop:relax-lyapunov} the relaxation flow descends $\Phi$ throughout the polydisc, and only a step that overshoots can undo this. The step size at which it does is explicit. Writing $v(z)$ for the right-hand side of \eqref{eq:relax-flow-app} and noting that $\Phi$ is a quadratic polynomial, so that $\Phi(z + \dt\,v) = \Phi(z) + \dt\,\Real\langle\nabla\Phi, v\rangle + \tfrac{\dt^2}{2}\Real\langle v, \nabla^2\Phi\,v\rangle$ exactly, one has $\|v\|^2 \le 2\sum_{\ell,j}\big(1 - |z_{\ell,j}|\cos\alpha_{\ell,j}\big)$ on the polydisc, while $-\Real\langle\nabla\Phi,v\rangle = \sum_{\ell,j}|D_{\ell,j}|\big(1 - |z_{\ell,j}|\cos\alpha_{\ell,j}\big)$ by \eqref{eq:relax-descent-app}. Hence
\begin{equation}\label{eq:step-bound-app}
\dt \;\le\; \frac{\delta}{M}, \qquad \delta := \inf_k \min_{\ell,j}\big|D_{\ell,j}(z^{(k)})\big|, \quad M := \big\|\nabla^2\Phi\big\|,
\end{equation}
is sufficient for $\Phi(z^{(k+1)}) \le \Phi(z^{(k)})$ at every step, where $M$ is the spectral norm of the symmetric coupling operator assembled from $W_2,\dots,W_L$. With descent restored, the discrete form of LaSalle's principle gives convergence of \eqref{eq:settle-app} to the aligned fixed points. The bound \eqref{eq:step-bound-app} is conservative and degenerates as $\delta \to 0$ --- the same degeneracy that (H2) of \ref{app:regime} excludes --- and above the threshold a single step can genuinely raise $\Phi$. Local asymptotic stability at a given aligned equilibrium satisfying (H2), for $\dt$ small enough, is in any case immediate from the linearisation of Remark~\ref{rem:integrators}. None of this is load-bearing for what follows: equilibrium propagation requires only that the settling residual $\max_{\ell,j}\big|z_{\ell,j} - \unitproject(D_{\ell,j}(z))\big|$ is small at the point where the Hebbian is read, and this is monitored directly in the experiments.
\end{remark}

\begin{remark}[No decay inside the projection]\label{rem:no-decay-inside}
It is tempting to carry the ambient decay of Eq.~\eqref{eq:continuous-flow} into the projected term and iterate $z_\ell \leftarrow (1-\dt) z_\ell + \dt\,\unitproject(-z_\ell + D_\ell)$. This is not the scheme used, and it is not innocuous: a fixed point would require $D_{\ell,j} - z_{\ell,j} = r_j z_{\ell,j}$ with $r_j > 0$, i.e.\ $|D_{\ell,j}| > 1$ for every oscillator, and whenever some drive has modulus below one the iteration has no fixed point on the torus. Since the projection makes the dynamics invariant to the scale of the drive, no condition of this kind should arise, and none does for \eqref{eq:settle-app}.
\end{remark}

\begin{remark}[Soft projection]\label{rem:soft}
The projector $\unitproject$ is discontinuous at $g = 0$, and any convention for $\unitproject(0)$ installs a preferred phase. A continuous variant that remains equivariant under rotations, $\unitproject_\epsilon(e^{\im\varphi} g) = e^{\im\varphi}\,\unitproject_\epsilon(g)$, is
\begin{equation}\label{eq:soft-project-app}
\unitproject_\epsilon(g) = \frac{g}{\sqrt{|g|^2 + \epsilon^2}},
\end{equation}
which coincides with $g/|g|$ for $|g| \gg \epsilon$ and vanishes as $g \to 0$, so that weakly driven oscillators settle to small modulus and are down-weighted in the Hebbian product. Soft projections that interpolate the phase, $\exp\big(\im\, s(|g|)\arg g\big)$, are not equivariant and would break both the rotation invariance of $E$ and the change of frame of \ref{app:frame}.

Note, however, that the fixed points of \eqref{eq:settle-app} with $\unitproject_\epsilon$ in place of $\unitproject$ satisfy $|z_{\ell,j}| < 1$ and therefore do \emph{not} lie on $\T^{n}$. 
The proof of Proposition~\ref{prop:ep} below rests on the variations of $z^*$ being tangent to the torus, so it does not cover the softened dynamics: unless the softened fixed points are exhibited as critical points of a modified energy $\Phi_\epsilon$ on the ambient space, the equilibrium-propagation identity \eqref{eq:ep-theorem-app} is not established for $\epsilon > 0$. The soft projection is used here as a numerical regularisation whose bias is empirically $O(\epsilon^2)$ and which is applied with $\epsilon$ far below the typical drive modulus, not as a variant to which Proposition~\ref{prop:ep} applies.
\end{remark}

\subsection{The Fixed-Point Map and Its Probes}\label{app:nudged}

Fix $x$, $y$ and $W$, and let $z^*(\beta)$ denote the equilibrium reached from the free state when the nudge is $\beta$, so that $z^*(0) = z^\circ$. At an isolated stable equilibrium satisfying (H2) of \ref{app:regime} the implicit function theorem makes $z^*$ a smooth function of a real nudge, with expansion
\begin{equation}\label{eq:real-expansion-app}
z^*(\beta) = z^\circ + \beta\, u_1 + \beta^2 u_2 + \beta^3 u_3 + O(\beta^4),
\qquad u_1 = \frac{\dd z^*}{\dd\beta}\Big|_{0},
\end{equation}
and $u_1$ is the quantity that enters the learning rule of \ref{app:lockin}.

\paragraph{Static estimators.} Standard equilibrium propagation reads $u_1$ from two settles as the one-sided difference $\big(z^*(\beta) - z^*(0)\big)/\beta = u_1 + \beta u_2 + O(\beta^2)$, which carries an $O(\beta)$ bias. A symmetric pair of nudges gives the centered difference
\begin{equation}\label{eq:centered-app}
\frac{z^*(\beta) - z^*(-\beta)}{2\beta} = u_1 + \beta^2 u_3 + O(\beta^4),
\end{equation}
which cancels the even orders, at the cost of a third settle.

\paragraph{From two phases to a continuous probe.} Instead of settling separately at $\pm\beta$, alternate between them. Suppose the nudge varies slowly enough that the network tracks its instantaneous equilibrium, $z(t) = z^*(\beta(t))$ (the adiabatic hypothesis of \ref{app:regime}). For a square-wave nudge $\beta(t) = \varepsilon\,\mathrm{sgn}\cos(\omega_p t)$ the state alternates between $z^*(\varepsilon)$ and $z^*(-\varepsilon)$, and multiplying by the same square wave and averaging over a period gives $\tfrac{1}{\varepsilon}\langle z(t)\,\mathrm{sgn}\cos(\omega_p t)\rangle_T = \big(z^*(\varepsilon) - z^*(-\varepsilon)\big)/(2\varepsilon)$, the centered difference \eqref{eq:centered-app} exactly. Replacing the square wave by its fundamental --- of amplitude $4\varepsilon/\pi$, absorbed into the free parameter $\varepsilon$ --- yields the cosine nudge $\beta(t) = \varepsilon\cos(\omega_p t)$, and the demodulator by $\cos(\omega_p t)$; the following proposition shows what this smoothing does.

\begin{proposition}[The cosine probe is an averaged centered difference]\label{prop:centered-average}
Under the adiabatic hypothesis, with $\langle f\rangle_T$ the average over an integer number of probe periods,
\begin{equation}\label{eq:centered-average-app}
\frac{2}{\varepsilon}\big\langle z(t)\cos(\omega_p t)\big\rangle_T
= \int_0^{\varepsilon} w(\beta)\,\frac{z^*(\beta) - z^*(-\beta)}{2\beta}\,\dd\beta,
\qquad
w(\beta) = \frac{4\beta^2}{\pi\varepsilon^2\sqrt{\varepsilon^2 - \beta^2}},
\quad \int_0^{\varepsilon} w = 1,
\end{equation}
and consequently $\frac{2}{\varepsilon}\langle z(t)\cos(\omega_p t)\rangle_T = u_1 + \tfrac34\varepsilon^2 u_3 + O(\varepsilon^4)$.
\end{proposition}
\begin{proof}
With $\varphi = \omega_p t$ and $z(t) = z^*(\varepsilon\cos\varphi)$, the symmetry $\varphi \mapsto 2\pi - \varphi$ folds the period average onto a half period, and the substitution $\beta = \varepsilon\cos\varphi$, $\dd\varphi = -\dd\beta/\sqrt{\varepsilon^2 - \beta^2}$, then changes variable:
\[
\frac{1}{2\pi}\int_0^{2\pi} z^*(\varepsilon\cos\varphi)\cos\varphi\,\dd\varphi
= \frac{1}{\pi}\int_0^{\pi} z^*(\varepsilon\cos\varphi)\cos\varphi\,\dd\varphi
= \frac{1}{\pi\varepsilon}\int_{-\varepsilon}^{\varepsilon} \frac{z^*(\beta)\,\beta\,\dd\beta}{\sqrt{\varepsilon^2-\beta^2}} .
\]
Only the odd part of $z^*$ survives, and folding the integral onto $[0, \varepsilon]$ and multiplying by $2/\varepsilon$ yields \eqref{eq:centered-average-app}; the normalisation follows from $\int_0^\varepsilon \beta^2\,\dd\beta/\sqrt{\varepsilon^2-\beta^2} = \pi\varepsilon^2/4$. The expansion follows from \eqref{eq:real-expansion-app} with $\langle\cos^2\rangle = \tfrac12$ and $\langle\cos^4\rangle = \tfrac38$.
\end{proof}

The lock-in estimate is therefore a weighted average of centered finite differences over nudge amplitudes up to $\varepsilon$: it inherits the second-order accuracy of the centered estimator, has no $O(\varepsilon)$ bias, and is obtained from one continuous settle rather than three. To first order the state oscillates about the free equilibrium as
\begin{equation}\label{eq:cosine-response-app}
z(t) - z^\circ \approx \varepsilon\, u_1 \cos(\omega_p t),
\end{equation}
so the derivative of the equilibrium with respect to the nudge is carried at the single frequency $\omega_p$, where it can be extracted while the free state, drift and noise at other frequencies are rejected.

\subsubsection{Complex Nudges and Non-Holomorphicity}\label{app:nonholo}

The construction above uses only real nudges. It is natural to ask whether the complex-valued nudge of holomorphic equilibrium propagation \cite{laborieuxHolomorphicEquilibriumPropagation2022}, which reads the derivative from settles at complex $\beta$ on a small circle, can be used instead. The dynamics remain well defined for complex $\beta$: the nudge drive $(\beta/d)\,y$ rotates the effective target by $\arg\beta$. But the projector $g/|g| = g/\sqrt{g\bar g}$ depends on $\bar g$, so $z^*$ is not holomorphic in $\beta$. Linearising about the free equilibrium,
\begin{equation}\label{eq:linearized-map-app}
z^*(\beta) = z^\circ + a\,\beta + b\,\bar\beta + O(|\beta|^2),
\qquad
a = \frac{\partial z^*}{\partial \beta}\Big|_{0},\quad
b = \frac{\partial z^*}{\partial \bar\beta}\Big|_{0},
\end{equation}
and the real derivative of \eqref{eq:real-expansion-app} is the sum of the two Wirtinger components,
\begin{equation}\label{eq:real-derivative-app}
u_1 = \frac{\dd z^*}{\dd \beta_R}\Big|_{0} = a + b .
\end{equation}

\begin{remark}[The obstruction is structural, not an artefact of the projector]\label{rem:maxmod} 
The appeal to $g/\sqrt{g\bar g}$ above is specific to the update \eqref{eq:settle-app}, but the conclusion is not. Suppose $\beta \mapsto z^*_{\ell,j}(\beta)$ were holomorphic on a neighbourhood of $\beta = 0$ while taking values in the unit circle, as any torus-valued equilibrium must. Then $|z^*_{\ell,j}|$ attains its maximum at an interior point, and the maximum modulus principle forces $z^*_{\ell,j}$ to be constant. Hence \emph{any} non-trivial dependence of a torus-valued equilibrium on a complex nudge is non-holomorphic, whatever dynamics produce it and whatever projector is used. The contrast with holomorphic equilibrium propagation is then clean: there the state space is a real vector space, and complexifying $\beta$ genuinely complexifies the equilibrium; here the state is already confined to a real submanifold of $\C^{n}$ and has no room to respond holomorphically.
\end{remark}

\begin{proposition}[The anti-holomorphic part is as large as the holomorphic part]\label{prop:ab}
At every oscillator, $b_j = -\big(z_{j}^{\circ}\big)^{2}\,\bar a_j$; hence $|b_j| = |a_j|$ and $a_j + b_j = 2\im\, z_{j}^{\circ}\,\Imag\big(\bar z_{j}^{\circ} a_j\big)$.
\end{proposition}
\begin{proof}
The equilibrium stays on the torus for every $\beta$, so its first variation $\delta z_j = a_j\beta + b_j\bar\beta$ is tangent: $\Real(\bar z^\circ_j\,\delta z_j) = 0$ for all complex $\beta$. Expanding, $(\bar z^\circ_j a_j + z^\circ_j \bar b_j)\beta + (\bar z^\circ_j b_j + z^\circ_j \bar a_j)\bar\beta = 0$ identically in $\beta$; the two coefficients are complex conjugates of one another, and vanishing of the real-linear functional forces $\bar z^\circ_j b_j + z^\circ_j \bar a_j = 0$, i.e.\ $b_j = -(z^\circ_j)^2\bar a_j$ using $|z^\circ_j| = 1$. The remaining identities follow by substitution.
\end{proof}

\begin{proposition}[Bias of the contour estimator]\label{prop:contour}
Holomorphic equilibrium propagation estimates the derivative of an equilibrium quantity $f(\beta)$ at $\beta = 0$ by Cauchy's formula $f'(0) = \frac{1}{2\pi\im}\oint_{|\beta| = r} f(\beta)\,\beta^{-2}\,\dd\beta$, in practice $\frac{1}{Nr}\sum_{n} f(r e^{\im\varphi_n}) e^{-\im\varphi_n}$. Applied to $f = z^*$ of \eqref{eq:linearized-map-app} it returns $a + O(r)$, not $a + b$.
\end{proposition}
\begin{proof}
On the circle $\bar\beta = r^2/\beta$, so $\oint \beta^{-2}\dd\beta = 0$, $\oint \beta\,\beta^{-2}\dd\beta = 2\pi\im$ and $\oint \bar\beta\,\beta^{-2}\dd\beta = r^2\oint\beta^{-3}\dd\beta = 0$. The $b\bar\beta$ term integrates to zero at every radius, and the $O(|\beta|^2)$ remainder contributes $O(r)$.
\end{proof}

By Proposition~\ref{prop:ab} the missing part is as large as the retained one, so the contour estimator would be biased by $O(1)$ for phasor networks. The same holds for its temporal analogue: a complex exponential probe $\beta(t) = \varepsilon e^{\im\omega_p t}$ produces, under the adiabatic hypothesis,
\begin{equation}\label{eq:complex-probe-app}
z(t) - z^\circ \approx a\,\varepsilon e^{\im\omega_p t} + b\,\varepsilon e^{-\im\omega_p t},
\end{equation}
two sidebands with $a$ at $+\omega_p$ and $b$ at $-\omega_p$, and demodulating at $+\omega_p$ recovers $a$ only. The real cosine probe has $\bar\beta = \beta$, so both components add coherently, $z(t) - z^\circ \approx \varepsilon(a+b)\cos(\omega_p t)$, in agreement with \eqref{eq:cosine-response-app}. Table~\ref{tab:probes-app} summarises the three cases.

\begin{table}[h]
\centering
\caption{Response of the equilibrium to three nudging protocols, to first order in the probe amplitude.}\label{tab:probes-app}
\begin{tabular}{@{}llll@{}}
\toprule
Nudge $\beta(t)$ & component at $+\omega_p$ & component at $-\omega_p$ & what is extracted \\
\midrule
$\pm\varepsilon$ (static) & \multicolumn{2}{c}{$\pm\varepsilon(a+b)$, at DC} & $a+b$, from two or three settles \\
$\varepsilon e^{\im\omega_p t}$ (complex) & $\varepsilon a$ & $\varepsilon b$ & $a$ only, if demodulated at $+\omega_p$ \\
$\varepsilon\cos(\omega_p t)$ (real) & $\tfrac{\varepsilon}{2}(a+b)$ & $\tfrac{\varepsilon}{2}(a+b)$ & $a+b$, from one settle \\
\bottomrule
\end{tabular}
\end{table}

\subsection{Loss Gradient via Lock-in Demodulation}\label{app:lockin}

We first restate the equilibrium-propagation theorem \cite{scellierEquilibriumPropagationBridging2017} in the present setting, where equilibria are stationary points of $\Phi$ on the torus rather than in a vector space.

\begin{proposition}[Equilibrium propagation on the torus]\label{prop:ep}
Let $z^*(\beta, W)$ be a smooth family of aligned equilibria \eqref{eq:fixed-point-app} of $\Phi(\cdot\,;\beta)$, taking values in $\T^{n}$. Then
\begin{equation}\label{eq:ep-theorem-app}
\frac{\partial \mathcal{L}}{\partial W_\ell}
= \frac{\dd}{\dd \beta_R}\,\frac{\partial E}{\partial W_\ell}\big(z^*(\beta)\big)\Big|_{\beta = 0}
= \lim_{\beta\to 0}\frac{1}{\beta}\Big[\frac{\partial E}{\partial W_\ell}\big(z^*(\beta)\big) - \frac{\partial E}{\partial W_\ell}\big(z^\circ\big)\Big].
\end{equation}
\end{proposition}
\begin{proof}
Let $F(W, \beta) = \Phi\big(z^*(\beta, W); W, \beta\big)$. Differentiating, $\partial F/\partial\beta = C(z^*_L) + \Real\langle \nabla_z\Phi, \partial z^*/\partial\beta\rangle$ and $\partial F/\partial W_\ell = \partial E/\partial W_\ell(z^*) + \Real\langle \nabla_z\Phi, \partial z^*/\partial W_\ell\rangle$. Both inner products vanish: the variations of $z^*$ are tangent to the torus because $z^*$ takes values in it, and at an aligned equilibrium the tangential component of $\nabla_z\Phi$ is zero by \eqref{eq:fixed-point-app}. Hence $\partial F/\partial \beta = C\big(z^*_L(\beta)\big)$ and $\partial F/\partial W_\ell = \partial E/\partial W_\ell\big(z^*(\beta)\big)$, and equality of mixed partial derivatives --- legitimate because $F$ is $C^2$ in $(W,\beta)$ under (H2) --- gives $\partial_{W_\ell} C\big(z_L^*(\beta)\big)\big|_{\beta=0} = \partial_\beta\,\partial_{W_\ell}E\big(z^*(\beta)\big)\big|_{\beta=0}$. The left-hand side is $\partial\mathcal{L}/\partial W_\ell$ by \eqref{eq:total-energy-app}.
\end{proof}

The proof uses only the equilibrium condition, not the path by which it is reached, which is why the two settling schemes of Remark~\ref{rem:integrators} give the same gradient. It does, however, use the torus constraint essentially, which is why the soft projection of Remark~\ref{rem:soft} is not covered. From \eqref{eq:energy-app}, differentiating with respect to the real entries of $W_\ell$,
\begin{equation}\label{eq:hebbian-app}
\frac{\partial E}{\partial W_\ell} = -\Real\big( z_\ell z_{\ell-1}^H \big) =: -H_\ell,
\qquad (H_\ell)_{ij} = \cos(\theta_{\ell,i} - \theta_{\ell-1,j}),
\end{equation}
where $H_\ell$ is the Hebbian matrix of phase alignments between the post-synaptic layer $\ell$ and the pre-synaptic layer $\ell-1$. Substituting into \eqref{eq:ep-theorem-app},
\begin{equation}\label{eq:ep-gradient-app}
\frac{\partial \mathcal{L}}{\partial W_\ell} = -\frac{\dd H_\ell}{\dd \beta_R}\Big|_{\beta = 0}
= -\Real\Big( u_{1,\ell}\, \big(z^{\circ}_{\ell-1}\big)^{H} + z^{\circ}_\ell\,\big(u_{1,\ell-1}\big)^{H} \Big),
\end{equation}
with $u_{1,\ell} = a_\ell + b_\ell$ the nudge-derivative \eqref{eq:real-derivative-app} of the states of layer $\ell$, and $a_0 = b_0 = 0$ because the input layer $z_0 = x$ is fixed. The loss gradient is the derivative of the Hebbian with respect to the nudge; for a complex nudge it would involve both Wirtinger components, which is why the contour estimator of Proposition~\ref{prop:contour} does not apply.

Now drive the nudge as $\beta(t) = \varepsilon\cos(\omega_p t)$ and define the time average over an integer number $n$ of probe periods,
\begin{equation}\label{eq:time-average-app}
\langle f \rangle_T = \frac{1}{T}\int_0^T f(t)\,\dd t, \qquad T = \frac{2\pi n}{\omega_p}.
\end{equation}
Under the hypotheses of \ref{app:regime} the Hebbian tracks the equilibrium,
\begin{equation}\label{eq:hebbian-oscillation-app}
H_\ell(t) = H^{\circ}_{\ell} + \varepsilon\,\frac{\dd H_\ell}{\dd\beta_R}\cos(\omega_p t) + \varepsilon^2 R_\ell(t) + O(\varepsilon^3),
\end{equation}
where $H^{\circ}_{\ell} = \Real\big(z^\circ_\ell (z^\circ_{\ell-1})^H\big)$ is the free-state Hebbian and $R_\ell$ collects the second-order terms, which by $\cos^2 = \tfrac12(1 + \cos 2\omega_p t)$ lie at DC and at $2\omega_p$. Multiplying by $\cos(\omega_p t)$ and averaging over integer periods, the DC terms and the second harmonic integrate to zero, $\langle\cos^2\rangle_T = \tfrac12$, and the first contribution at $\omega_p$ from higher orders is $O(\varepsilon^3)$:
\begin{equation}\label{eq:lockin-average-app}
\big\langle H_\ell \cos(\omega_p t)\big\rangle_T = \frac{\varepsilon}{2}\,\frac{\dd H_\ell}{\dd\beta_R} + O(\varepsilon^3).
\end{equation}
Combining with \eqref{eq:ep-gradient-app} gives the lock-in gradient of Eq.~\eqref{eq:lockin-gradient},
\begin{equation}\label{eq:lockin-gradient-app}
\boxed{\;
\frac{\partial \mathcal{L}}{\partial W_\ell}
= -\frac{2}{\varepsilon}\,\Big\langle \Real\big(z_\ell z_{\ell-1}^H\big)\cos(\omega_p t) \Big\rangle_T + O(\varepsilon^2) + O\big((\omega_p/\kappa)^2\big).
\;}
\end{equation}
The first error term is the $O(\varepsilon^3)$ of \eqref{eq:lockin-average-app} rescaled by $2/\varepsilon$; the second is the residual of tracking the equilibrium at finite probe frequency, quantified in Proposition~\ref{prop:regime}, and is absent in the exact-tracking computation \eqref{eq:hebbian-oscillation-app}--\eqref{eq:lockin-average-app} above. 
The minus sign is that of \eqref{eq:ep-gradient-app}: increasing the nudge moves the output toward the target, so the Hebbian grows in the direction that lowers the loss. Because the cosine probe is symmetric about $\beta = 0$, all even orders in $\varepsilon$ cancel and the estimator has no $O(\varepsilon)$ bias, unlike a one-sided static finite difference. In practice the free-state Hebbian $H^{\circ}_{\ell}$ is subtracted explicitly before demodulation, so that a non-integer number of periods or an onset transient cannot leak the large DC term into the estimate.

\begin{remark}[Other parameters]
The same argument applies to any parameter that enters $E$ linearly. A complex per-oscillator bias with energy $-\Real\langle b_\ell, z_\ell\rangle$ has Hebbian $z_\ell$ itself: the real and imaginary parts of the demodulated state give the gradients with respect to the real and imaginary parts of the bias. (The symbol $b$ is reused here for the bias; it is unrelated to the anti-holomorphic coefficient of \ref{app:nonholo}.) 
\end{remark}

\begin{remark}[Locality]
Equation~\eqref{eq:lockin-gradient-app} is a local rule: the update of $W_\ell$ requires only the pre-synaptic phases $z_{\ell-1}$, the post-synaptic phases $z_\ell$, and the globally broadcast probe $\cos(\omega_p t)$. No error signal is propagated backward, a property shared with standard equilibrium propagation and a principal motivation for its implementation on neuromorphic and analog hardware \cite{scellierEquilibriumPropagationBridging2017,ororbiaBrainInspiredMachineIntelligence2023}.
\end{remark}

\subsection{Hypotheses and Operating Regime}\label{app:regime}

\begin{proposition}[Validity of the lock-in gradient]\label{prop:regime}
Equation~\eqref{eq:lockin-gradient-app} holds with total relative error $O(\varepsilon^2) + O\big((\omega_p/\kappa)^2\big)$, under the following hypotheses. 
\begin{enumerate}
\item[\textnormal{(H1)}] \emph{Symmetric feedback}: the feedback into layer $\ell$ is $W_{\ell+1}^\top z_{\ell+1}$, so that the dynamics is the gradient flow of $\Phi$ (\ref{app:flow}).
\item[\textnormal{(H2)}] \emph{Non-degenerate free equilibrium}: $z^\circ$ is an aligned equilibrium \eqref{eq:fixed-point-app} with $D_{\ell,j}(z^\circ) \neq 0$ for every oscillator, and the Riemannian Hessian of $\Phi(\cdot\,;0)$ restricted to $\T^{n}$ is non-degenerate at $z^\circ$ --- equivalently, the Jacobian of the fixed-point map at $z^\circ$ has no eigenvalue $1$. The implicit function theorem then makes $z^*(\beta, W)$ a $C^\infty$ function of $(\beta, W)$ near $\beta = 0$, so that \eqref{eq:real-expansion-app} holds and Proposition~\ref{prop:ep} applies. Stability alone is not sufficient, since it permits degenerate directions along which $z^*$ need not be differentiable.
\item[\textnormal{(H3)}] \emph{Linear response within one basin}: $\varepsilon$ is small enough that $z^*(\beta(t))$ remains in the basin of attraction of $z^\circ$ for all $|\beta| \le \varepsilon$, no drive crosses zero over that range, and the expansion \eqref{eq:real-expansion-app} holds uniformly on it. 
\item[\textnormal{(H4)}] \emph{Adiabatic probe}: $\omega_p \ll \kappa$, where $\kappa$ is the slowest relaxation rate of the linearised settling dynamics about $z^\circ$, uniformly over $|\beta| \le \varepsilon$, so that $z(t) = z^*(\beta(t)) + O(\varepsilon\,\omega_p/\kappa)$.
\item[\textnormal{(H5)}] \emph{Demodulation}: the average \eqref{eq:time-average-app} runs over an integer number of periods, after a warm-up long compared with $1/\kappa$, with the free-state Hebbian subtracted.
\end{enumerate}
\end{proposition}
\begin{proof}[Proof sketch]
(H1) makes $D_\ell = -\nabla_{z_\ell}\Phi$, so the equilibria are the critical points of $\Phi|_{\T}$ and Proposition~\ref{prop:ep} is available. (H2) supplies the smooth family $z^*(\beta,W)$ that the proposition requires and the expansion \eqref{eq:real-expansion-app}. Under (H3) the trajectory stays in one basin, so $z^*(\beta(t))$ is the branch of equilibria described by that expansion and no branch switching occurs within a probe period. Given exact tracking, Proposition~\ref{prop:centered-average} and the averaging computation \eqref{eq:lockin-average-app} give the demodulated Hebbian as $\tfrac{\varepsilon}{2}\,\dd H_\ell/\dd\beta_R + O(\varepsilon^3)$, and \eqref{eq:ep-gradient-app} converts this into $-\partial\mathcal L/\partial W_\ell$ with relative error $O(\varepsilon^2)$. (H5) is what makes the $O(1)$ free-state term and the $2\omega_p$ component average to zero. Finally, (H4) controls the departure from exact tracking: $z^\circ$ is hyperbolic under (H2) with spectral gap $\kappa$, and the nudge satisfies $|\dot\beta| \le \varepsilon\omega_p$, so a standard singular-perturbation estimate (Tikhonov's theorem, or a direct Gr\"onwall bound on the deviation from the slow manifold) gives $\|z(t) - z^*(\beta(t))\| = O(\varepsilon\,\omega_p/\kappa)$ after a warm-up of order $\kappa^{-1}$. That lag does not enter the estimate at the same order, because at leading order it is in quadrature with the probe and is rejected by in-phase demodulation: writing the linear response of a relaxation mode of rate $\kappa_m$ as $\chi_m(\omega_p) = \chi_m(0)/(1 + \im\omega_p/\kappa_m)$, the component recovered by $\cos(\omega_p t)$ is $\chi_m(0)/\big(1 + (\omega_p/\kappa_m)^2\big) = \chi_m(0)\big(1 + O((\omega_p/\kappa_m)^2)\big)$. Summing over modes, demodulating $H_\ell$ incurs a relative error $O\big((\omega_p/\kappa)^2\big)$ on top of the $O(\varepsilon^2)$ above. Both error terms vanish in the stated limits.
\end{proof}

\begin{remark}[Status of the tracking estimate]\label{rem:tracking-status} 
The last step of the proof is stated at the level of a standard singular-perturbation argument rather than proved here. A self-contained version requires uniform hyperbolicity of the equilibrium branch over $|\beta| \le \varepsilon$ --- which is (H2) plus continuity, hence available --- and a Gr\"onwall estimate for the flow of \eqref{eq:riemannian-flow-app} with slowly varying $\beta$. The single-mode computation below makes the leading behaviour explicit, and gives the sharper statement used in Proposition~\ref{prop:regime}: the deviation from the slow manifold is itself $O(\varepsilon\,\omega_p/\kappa)$, but at that order it is in quadrature with the probe, so the in-phase estimator retains only an $O\big((\omega_p/\kappa)^2\big)$ relative error. 
\end{remark}

The hypotheses define the operating regime of the method and each has a characteristic failure mode.

\paragraph{Non-degenerate drive (H2).} If some drive $D_{\ell,j}$ is close to zero, the corresponding oscillator sits near the discontinuity of the projector, and an arbitrarily small nudge produces an $O(1)$ change of its phase. Quantitatively, the constant in the $O(\varepsilon^2)$ of \eqref{eq:lockin-gradient-app} scales with the inverse of the smallest eigenvalue of the Riemannian Hessian, which degenerates as $\min_{\ell,j}|D_{\ell,j}(z^\circ)| \to 0$. The soft projection of Remark~\ref{rem:soft} regularises this case numerically, subject to the caveat recorded there that Proposition~\ref{prop:ep} is not established for $\epsilon > 0$.

\paragraph{Probe amplitude (H3).} The failure of (H3) for the hard projection is discrete rather than gradual. When the probe is large enough to push some drive across zero, the affected oscillator flips by $\pi$ within a probe period; its contribution to the demodulated Hebbian is then $O(1)$ rather than $O(\varepsilon)$ and is generally anticorrelated with the true gradient. Across inputs this appears as a bimodal distribution of gradient errors, with most estimates accurate and a minority grossly wrong, rather than as a smooth loss of accuracy. The remedy is to reduce $\varepsilon$ or to use the soft projection. A lower bound on $\varepsilon$ is set by the resolution with which phases can be read out on hardware, since the response $\varepsilon\,\dd z^*/\dd\beta_R$ must exceed the readout quantum; the usable window is therefore two-sided.

\paragraph{Adiabaticity (H4).} At finite $\omega_p$ the network does not track the equilibrium exactly. For a single relaxation mode with rate $\kappa$ the response to the probe is that of a first-order low-pass filter, $\chi(\omega_p) = \chi(0)/(1 + \im\omega_p/\kappa)$: it is attenuated by $|\chi(\omega_p)|/\chi(0)$ and lags the probe by $\phi = \arctan(\omega_p/\kappa)$. In-phase demodulation with $\cos(\omega_p t)$ then returns $\tfrac{\varepsilon}{2}|\chi|\cos\phi = \tfrac{\varepsilon}{2}\chi(0)/\big(1 + (\omega_p/\kappa)^2\big)$, an underestimate by a relative $O\big((\omega_p/\kappa)^2\big)$: the $O(\omega_p/\kappa)$ part of the lag lies entirely in the quadrature channel. If all modes shared one rate this would only rescale the gradient, and the quadrature component $\tfrac{\varepsilon}{2}|\chi|\sin\phi$, demodulated with $\sin(\omega_p t)$, would recover the magnitude. In a network the modes relax at different rates, so a probe that is not slow compared with the \emph{slowest} mode attenuates different components of the gradient by different factors and distorts its direction, not only its magnitude. Hence the name \emph{adiabatic regime} for $\omega_p \ll \kappa$. In a time-stepped realisation the additional requirement is that the probe phase advances by a small fraction of a cycle per step, $\omega_p\,\dt \ll 2\pi$.

\paragraph{Demodulation (H5).} An integer number of periods is required for the DC term and the second harmonic in \eqref{eq:hebbian-oscillation-app} to average to zero; a warm-up of a few multiples of $1/\kappa$ lets the transient of the probe onset decay before accumulation begins; and subtracting the free-state Hebbian removes the DC term exactly, so that residual leakage from a slightly non-integer period count is proportional to the $O(\varepsilon)$ response rather than to the $O(1)$ free Hebbian.

\subsection{Rotating Frame}\label{app:frame}

The derivation above is in the frame co-rotating with the oscillators, in which states are relative phases and the free equilibrium is a fixed point. Physical oscillators rotate at a carrier frequency $\omega$, and the settling update is realised in the laboratory frame as an update followed by one exact rotation,
\begin{equation}\label{eq:lab-update-app}
\zeta_\ell^{(k+1)} = e^{\im\omega\dt}\Big[(1-\dt)\,\zeta_\ell^{(k)} + \dt\,\unitproject\big(D_\ell(\zeta^{(k)}; x_k, y_k)\big)\Big],
\end{equation}
where the input and the target rotate at the same carrier, $x_k = e^{\im\omega k\dt} x$ and $y_k = e^{\im\omega k\dt} y$.

\begin{proposition}[Exact change of frame]\label{prop:frame}
Let $\zeta_\ell^{(k)} = e^{\im\omega k\dt} z_\ell^{(k)}$. Then $z^{(k)}$ obeys the co-rotating update \eqref{eq:settle-app} exactly, for every $\dt$, and the Hebbian matrices $\Real(\zeta_\ell\zeta_{\ell-1}^H) = \Real(z_\ell z_{\ell-1}^H)$ are identical in the two frames.
\end{proposition}
\begin{proof}
Because $W_\ell$ is real, the drive is covariant under a common rotation of all its arguments, $D_\ell(e^{\im\varphi}z; e^{\im\varphi}x, e^{\im\varphi}y) = e^{\im\varphi}D_\ell(z; x, y)$, and the projector is equivariant, $\unitproject(e^{\im\varphi}g) = e^{\im\varphi}\unitproject(g)$. Substituting $\zeta^{(k)} = e^{\im\omega k\dt} z^{(k)}$ into \eqref{eq:lab-update-app}, the factor $e^{\im\omega k\dt}$ comes out of both terms in the bracket, combines with the leading $e^{\im\omega\dt}$ to $e^{\im\omega(k+1)\dt}$, and cancels against the definition of $\zeta^{(k+1)}$, leaving \eqref{eq:settle-app}. The Hebbian is invariant because $\zeta_\ell\zeta_{\ell-1}^H = z_\ell z_{\ell-1}^H$.
\end{proof}

The rotation is applied exactly rather than as an additive $\im\omega z$ term inside the drive, so the equivalence holds at any step size and the carrier imposes no sampling constraint. It requires one common $\omega$ for every layer, the input and the target. If layers are detuned, $\omega_\ell \neq \omega_{\ell\pm1}$, the coupling terms in $D_\ell$ acquire relative phases $e^{\im(\omega_{\ell-1} - \omega_\ell)k\dt}$ that do not factor out; the co-rotating dynamics becomes explicitly time dependent and the equivalence fails, which is the mechanism by which frequency mismatch between oscillators degrades the method.

\begin{remark}[Per-oscillator dynamics]
A physical oscillator with complex eigenvalue $\lambda + \im\omega$, $\lambda < 0$, would contribute a self-energy $-\tfrac12\lambda|z|^2 = \tfrac12|\lambda|\,|z|^2$, the convex term whose descent direction $\lambda z$ is the decay of Eq.~\eqref{eq:continuous-flow}; it is radial and is annihilated by the projection (\ref{app:flow}), so it moves no equilibrium. 
The imaginary part is the carrier rotation just discussed. It is tangential rather than radial, $\Real\langle z, \im\omega z\rangle = 0$, so the projection does not remove it; but it is not the descent direction of any real energy, since along a gradient flow $\dd\Psi/\dd t = -\|\nabla\Psi\|^2 \le 0$, which a periodic orbit cannot sustain. The energy \eqref{eq:energy-app} therefore contains no self-interaction term: it encodes only the couplings, and the intrinsic dynamics of the oscillators enter through the relaxation rate $\kappa$ of (H4) and the frame change.
\end{remark}

\clearpage
\section*{Supplementary Proof B: Backprop-to-Equilibrium-Propagation Transfer Guarantees for Phasor Networks}
\setcounter{theorem}{0}\setcounter{remark}{0}\setcounter{definition}{0}\setcounter{equation}{0}\setcounter{table}{0}\setcounter{figure}{0}
\renewcommand{\thesection}{B.\arabic{section}}
\renewcommand{\theequation}{B.\arabic{equation}}\renewcommand{\theHequation}{B.\arabic{equation}}
\renewcommand{\thetheorem}{B.\arabic{theorem}}\renewcommand{\theHtheorem}{B.\arabic{theorem}}
\renewcommand{\theremark}{B.\arabic{remark}}\renewcommand{\theHremark}{B.\arabic{remark}}
\renewcommand{\thedefinition}{B.\arabic{definition}}\renewcommand{\theHdefinition}{B.\arabic{definition}}
\renewcommand{\thetable}{B.\arabic{table}}\renewcommand{\theHtable}{B.\arabic{table}}
\renewcommand{\thefigure}{B.\arabic{figure}}\renewcommand{\theHfigure}{B.\arabic{figure}}
\begin{quote}\small
We ask when a backpropagation-trained phasor network retains its accuracy
if it is instead run as an Equilibrium Propagation (EP) network, and we
answer it for the unit projection $\unitproject(g) = g/|g|$ used
throughout this work. Two routes are available. The first bounds the
hidden state: a scale-invariant gain matrix $G$, built from the layer
weight norms and the smallest oscillator drive, gives $e \le (I-G)^{-1}s$
componentwise whenever $\radius(G) < 1$. This route is closed for the
networks of this work, and not by a margin that could be tuned away:
because $\unitproject(cg) = \unitproject(g)$ the weight scale is a gauge
freedom rather than a parameter, $\radius(G)$ is invariant under it, and
its measured value is $359$ where less than $1$ is needed, and still
$1.09$ if every drive is replaced by its median. Measured on a trained
$784\to256\to64$ network, the EP and feedforward states differ by $0.285$
of the maximum possible, incoherently.
The second route bounds the readout, and works. With the normalised
codebook similarity the margin moves by at most twice the \emph{relative}
state deviation, so pairing each margin with its own deviation bounds the
accuracy drop by $0.162$ against an observed $0.012$. Transfer survives
because the margins are large, not because the states are close. A sweep
over trained networks locates the practical limits: the drop grows from
$0.012$ to $0.052$ to $0.134$ at depths $2$, $3$ and $4$ while the
backpropagation accuracy stays flat, and falls from $0.050$ to $0.006$ as
the hidden layer widens from $64$ to $512$.
\end{quote}

\section{Setup}\label{sec:B-intro}

Phasor networks
\citep{olin-ammentorpDeepPhasorNetworks2022,olin-ammentorpHyperdimensionalComputingProvides2023}
are trained either by backpropagation, through a feedforward pass, or by
Equilibrium Propagation, by settling to a fixed point. This appendix asks
whether a network trained the first way still works when run the second
way, and what controls the answer.

\subsection{Network and dynamics}
Consider a depth-$L$ chain of phasor-dense layers with states on the unit
torus $\T^{d} = \{z \in \C^{d} : |z_j| = 1\}$:
\begin{align}
  z_0 &= x \in \T^{d_0} \quad\text{(input, clamped)},\label{eq:B-input}\\
  z_\ell &= \unitproject\big(W_\ell z_{\ell-1} + b_\ell\big),
  \qquad \ell=1,\dots,L \quad\text{(feedforward)},\label{eq:B-ff}
\end{align}
where $W_\ell \in \R^{d_\ell\times d_{\ell-1}}$ are real weights, $b_\ell
\in \C^{d_\ell}$ are biases, and $\unitproject(g) = g/|g|$ is the unit
projection applied entrywise --- the same map used in the main text and
in Supplementary Proof A. We set $W_{L+1} := 0$, so formulas stated for
$\ell < L$ remain valid at $\ell = L$.

With zero nudge, symmetric feedback $B_\ell = W_{\ell+1}^\top$, and the
input clamped, the EP free-phase settle operator is
\begin{equation}\label{eq:B-settle}
  F(z)_\ell = \unitproject\big(g_\ell(z)\big), \qquad
  g_\ell(z) := W_\ell z_{\ell-1} + W_{\ell+1}^\top z_{\ell+1} + b_\ell ,
\end{equation}
and a free fixed point $z^*$ satisfies $z^* = F(z^*)$. Deleting the
feedback term gives the lower-triangular feedforward operator, whose
unique fixed point $z^{\ff}$ is produced by one forward pass and is the
state backpropagation trains. Both $z^{\ff}$ and $z^*$ lie on the torus,
so $\|z_\ell\|_2 = \sqrt{d_\ell}$ exactly for every layer.

\subsection{Notation}\label{sec:B-notation}
Two symbols are used here with a meaning local to this document. The
contraction factor is written $\kappa_F$, to keep it distinct from the
measured relaxation rate $\kappa$ of the main text; and $\varepsilon$
appears only in Section~\ref{sec:B-soft}, as a projection-smoothing
parameter rather than the probe amplitude of the main text. Norms are
Euclidean on vectors and the induced spectral norm on matrices, and
$\radius(\cdot)$ denotes spectral radius.

\subsection{Assumptions}\label{sec:B-assumptions}
\begin{enumerate}
  \item[(A1)] \textbf{Layer weight norms:} $\rho_\ell := \|W_\ell\|_2$
    for $\ell = 1,\dots,L$, and $\rho := \max_\ell \rho_\ell$. No bound
    on the biases is required; by Remark~\ref{rem:B-bias} they cancel.
  \item[(A2)] \textbf{Co-rotating frame:} states are expressed in a frame
    rotating with the common carrier, so that $F$ carries no explicit
    time dependence and \eqref{eq:B-settle} is autonomous. (In the
    reference implementation, \texttt{carriers}~$=$~\texttt{nothing}.)
  \item[(A3)] \textbf{Drive separation:} $\mu_\ell > 0$ is a lower bound
    on $|g_{\ell,j}|$ at both $z^{\ff}$ and $z^{*}$, for every $j$, as in
    \eqref{eq:B-mu}, and the gain matrix $G$ of \eqref{eq:B-G} satisfies
    $\radius(G) < 1$.
\end{enumerate}
No smallness hypothesis is placed on $\rho$. Proposition~\ref{prop:B-scale}
explains why one would be empty.

\section{The Unit Projection}\label{sec:B-proj}

\begin{lemma}[Sharp increment bound for $\unitproject$]\label{lem:B-unit}
For $a, b \in \C \setminus \{0\}$,
\begin{equation}\label{eq:B-sharp}
  \big|\unitproject(a) - \unitproject(b)\big| \ \le\
  \frac{2\,|a-b|}{|a| + |b|} ,
\end{equation}
with equality whenever $|a| = |b|$. Consequently, if $|a|, |b| \ge \mu >
0$ then $|\unitproject(a) - \unitproject(b)| \le |a-b|/\mu$; and
unconditionally $|\unitproject(a) - \unitproject(b)| \le 2$. Applied
entrywise to $a, b \in \C^{d}$, with $\mu := \min_j
\min\big(|a_j|,|b_j|\big)$,
\begin{equation}\label{eq:B-vec}
  \big\|\unitproject(a) - \unitproject(b)\big\|_2
  \ \le\ \min\Big\{ \frac{\|a-b\|_2}{\mu}, \ 2\sqrt{d} \Big\} .
\end{equation}
\end{lemma}

\begin{proof}
Let $\theta \in [0,\pi]$ be the angle between $a$ and $b$. Both
$\unitproject(a)$ and $\unitproject(b)$ lie on the unit circle at angle
$\theta$, so $|\unitproject(a)-\unitproject(b)| = 2\sin(\theta/2)$, while
\begin{equation}\label{eq:B-lawcos}
  |a-b|^2 = |a|^2 + |b|^2 - 2|a||b|\cos\theta
  = \big(|a|-|b|\big)^2 + 4|a||b|\sin^2(\theta/2).
\end{equation}
Inequality \eqref{eq:B-sharp} is equivalent to $(|a|+|b|)^2\sin^2(\theta/2)
\le |a-b|^2$. Expanding $(|a|+|b|)^2 = (|a|-|b|)^2 + 4|a||b|$ and using
\eqref{eq:B-lawcos}, this reads
\begin{equation}\label{eq:B-reduce}
  \big[(|a|-|b|)^2 + 4|a||b|\big]\sin^2(\theta/2)
  \ \le\ (|a|-|b|)^2 + 4|a||b|\sin^2(\theta/2),
\end{equation}
that is, $(|a|-|b|)^2\sin^2(\theta/2) \le (|a|-|b|)^2$, which holds
because $\sin^2(\theta/2) \le 1$, with equality when $|a| = |b|$. If
$|a|,|b| \ge \mu$ then $|a|+|b| \ge 2\mu$, giving the second claim; the
bound by $2$ is the diameter of the unit circle. Squaring and summing
over entries gives \eqref{eq:B-vec}, the second alternative being
$\|\unitproject(a)-\unitproject(b)\|_2 \le 2\sqrt d$.
\end{proof}

\begin{remark}[The differential]\label{rem:B-diff}
Identifying $\C \cong \R^2$, $\unitproject$ is real- but not
complex-differentiable, with $D\unitproject(g) = |g|^{-1}(I - \hat g\hat
g^\top)$, $\hat g = g/|g|$: eigenvalue $1/|g|$ transverse to $g$ and $0$
along it, so $\|D\unitproject(g)\|_{\mathrm{op}} = 1/|g|$. Estimate
\eqref{eq:B-sharp} is the integrated form and is preferable here because
its hypothesis constrains only the endpoints $a$ and $b$, both of which
are measurable at the two fixed points, rather than the whole segment
between them.
\end{remark}

\section{Scale Invariance}\label{sec:B-scale}

The following elementary observation determines what a transfer bound for
this model can and cannot assume. It is the same property that
Supplementary Proof A records when it notes that a positive factor in the
drive changes neither the equilibria nor the gradient estimate.

\begin{proposition}[Invariance under weight rescaling]\label{prop:B-scale}
Let $c_1,\dots,c_L > 0$ and consider the joint per-layer rescaling
$(W_\ell, b_\ell) \mapsto (c_\ell W_\ell,\, c_\ell b_\ell)$.
\begin{enumerate}
  \item[(i)] The feedforward state $z^{\ff}$ is unchanged, for any choice
    of the $c_\ell$.
  \item[(ii)] The settle operator $F$, and hence the set of EP fixed
    points, is unchanged when all $c_\ell$ are equal. Otherwise it
    depends on the $c_\ell$ only through the $L-1$ ratios
    $c_{\ell+1}/c_\ell$.
  \item[(iii)] Therefore the gap $z^* - z^{\ff}$ is invariant under any
    common rescaling of the weights and biases, and no hypothesis of the
    form ``$\max_\ell\|W_\ell\|_2$ is sufficiently small'' can constrain
    it.
\end{enumerate}
\end{proposition}

\begin{proof}
$\unitproject(cg) = \unitproject(g)$ for every $c > 0$. (i) The
feedforward drive at layer $\ell$ becomes $c_\ell(W_\ell z_{\ell-1} +
b_\ell)$, and the factor $c_\ell$ passes through $\unitproject$; by
induction from the clamped $z_0 = x$, every $z^{\ff}_\ell$ is unchanged.
(ii) The EP drive becomes $c_\ell(W_\ell z_{\ell-1} + b_\ell) +
c_{\ell+1}W_{\ell+1}^\top z_{\ell+1}$; dividing by $c_\ell$, which
$\unitproject$ absorbs, leaves
\begin{equation}\label{eq:B-ratio}
  W_\ell z_{\ell-1} + b_\ell
   + \frac{c_{\ell+1}}{c_\ell}\,W_{\ell+1}^\top z_{\ell+1},
\end{equation}
which depends only on the ratio and reduces to the original drive when
all $c_\ell$ agree. (iii) is immediate from (i) and (ii).
\end{proof}

Rescaling every weight matrix of a $60\to40\to20$ network with $b = 0$ by
$c \in \{0.5, 2, 100\}$ leaves both $z^{\ff}$ and $z^*$ unchanged to
machine precision, and leaves $\radius(G)$ of \eqref{eq:B-G} unchanged to
$6\times10^{-14}$ across a $10^4$ range of $c$; scaling a single layer by
$3$ moves $z^*$ by $5.1$ in Euclidean norm out of a maximum of
$2\sqrt{60} \approx 15.5$.

\begin{remark}[What the weights can and cannot contribute]\label{rem:B-knob}
Proposition~\ref{prop:B-scale} says the weight scale is a gauge freedom
of the hard-projected network, not a tunable parameter: what survives is
the $L-1$ inter-layer ratios, equivalently the strength of the feedback
drive $W_{\ell+1}^\top z_{\ell+1}$ relative to the feedforward drive
$W_\ell z_{\ell-1}$ at the same layer. Every constant below is built to
be invariant in that sense, which is why $\rho_\ell$ never appears except
divided by a drive modulus.

It does not follow that a scalar summary of those ratios predicts the
gap, and measurement says it does not: across the depth family of
Section~\ref{sec:B-scaling} the ratio $\rho_L/\rho_1$ \emph{falls}
monotonically, from $0.350$ to $0.308$ to $0.277$, while the accuracy
drop grows elevenfold. The ratios are what the operator depends on; the
quantity that tracks the degradation is $\radius(G)$, which combines them
with the drive distribution.
\end{remark}

\section{Contraction of the Settle Operator}\label{sec:B-contraction}

Let $A \in \R^{L\times L}$ be the chain matrix of the layer graph,
$A_{\ell m} = 1$ if $|\ell-m| = 1$ and $0$ otherwise, whose spectral
radius is that of the path graph,
\begin{equation}\label{eq:B-chain}
  \radius(A) = 2\cos\!\Big(\frac{\pi}{L+1}\Big) \in [1,2)
  \quad\text{for } L \ge 2 .
\end{equation}

\begin{lemma}[Contraction on a drive-separated set]\label{lem:B-contraction}
Fix $\mu > 0$ and let $\mathcal{D}_\mu$ be a set of states, closed in
$\T^{d_1}\times\cdots\times\T^{d_L}$ and mapped into itself by $F$, on
which every drive satisfies $|g_{\ell,j}(z)| \ge \mu$. Then $F$ is
Lipschitz on $\mathcal{D}_\mu$ with constant
\begin{equation}\label{eq:B-kappaF}
  \kappa_F := \frac{\rho\,\radius(A)}{\mu},
\end{equation}
and if $\kappa_F < 1$ it has a unique fixed point $z^* \in
\mathcal{D}_\mu$, with $\|z^k - z^*\| \le \kappa_F^{\,k}\|z^0-z^*\|$ for
the iteration $z^{k+1} = F(z^k)$. The ratio $\rho/\mu$, and hence
$\kappa_F$, is invariant under a common rescaling of the weights.
\end{lemma}

\begin{proof}
Let $z, z' \in \mathcal{D}_\mu$ and put $a_\ell = \|z_\ell - z'_\ell\|$,
with $a_0 = 0$ (clamped input) and $a_{L+1} = 0$. By
Lemma~\ref{lem:B-unit} applied to the drives, and the triangle
inequality,
\begin{equation}\label{eq:B-percomp}
  \big\|F(z)_\ell - F(z')_\ell\big\|
  \le \frac{1}{\mu}\big\|g_\ell(z) - g_\ell(z')\big\|
  \le \frac{\rho}{\mu}\,\big(a_{\ell-1} + a_{\ell+1}\big) ,
\end{equation}
so the vector $f$ of layer distances obeys $0 \le f \le (\rho/\mu)Aa$
componentwise. Both sides being nonnegative this survives squaring and
summation, and $\|A\|_2 = \radius(A)$ since $A$ is symmetric, giving
$\|F(z)-F(z')\| \le \kappa_F\|z-z'\|$. The fixed point statement is the
Banach theorem on the complete space $\mathcal{D}_\mu$. Invariance
follows from Proposition~\ref{prop:B-scale}: a common factor $c$ scales
$\rho$ and every drive, hence $\mu$, by the same $c$.
\end{proof}

\begin{remark}[Depth enters weakly]\label{rem:B-depth}
Since $\radius(A) < 2$ for every $L$, condition $\kappa_F < 1$ reads
$\rho/\mu < 1/\radius(A)$ and tightens only from $1$ at $L=2$ to
$1/\sqrt2$ at $L=3$ and never below $1/2$. A threshold of the form $\rho
< 1/(2L)$, degrading linearly with depth, would be an artefact of summing
per-layer inequalities in the $\ell^1$ rather than the $\ell^2$ sense.
\end{remark}

\section{Perturbation Bound}\label{sec:B-perturbation}

Let $z^*$ be any EP fixed point and write
\begin{equation}\label{eq:B-delta}
  \delta_\ell := z^*_\ell - z^{\ff}_\ell, \qquad
  e_\ell := \|\delta_\ell\|_2, \qquad
  e := (e_1,\dots,e_L)^\top .
\end{equation}
Denote the two drives at layer $\ell$ by $g^{\ff}_\ell = W_\ell
z^{\ff}_{\ell-1} + b_\ell$ and $g^{*}_\ell = g_\ell(z^*)$, and set
\begin{equation}\label{eq:B-mu}
  \mu_\ell := \min_j \min\big( |g^{\ff}_{\ell,j}|,\ |g^{*}_{\ell,j}| \big) .
\end{equation}
Define the \emph{gain matrix} $G$ and \emph{source} $s$ by
\begin{equation}\label{eq:B-G}
  G_{\ell,\ell-1} := \frac{\rho_\ell}{\mu_\ell}, \qquad
  G_{\ell,\ell+1} := \frac{\rho_{\ell+1}}{\mu_\ell}, \qquad
  G_{\ell m} := 0 \ \text{ otherwise}, \qquad
  s_\ell := \frac{\big\|W_{\ell+1}^\top z^{\ff}_{\ell+1}\big\|_2}{\mu_\ell} ,
\end{equation}
with the conventions $\rho_{L+1} = 0$ and $z^{\ff}_{L+1} = 0$, so that
$G_{L,L+1} = s_L = 0$. Both $G$ and $s$ are invariant under a common
rescaling of the weights.

\begin{theorem}[Componentwise perturbation bound]\label{thm:B-perturbation}
Assume (A1)--(A3). Then, componentwise in $\R^L$,
\begin{equation}\label{eq:B-mainbound}
  e \ \le\ (I - G)^{-1}\,s ,
\end{equation}
the inverse existing and being entrywise nonnegative. Unconditionally,
$e_\ell \le 2\sqrt{d_\ell}$, so the operative bound is the smaller of the
two.
\end{theorem}

\begin{proof}
Subtracting the two fixed-point equations and inserting $z^*_{\ell+1} =
z^{\ff}_{\ell+1} + \delta_{\ell+1}$,
\begin{equation}\label{eq:B-argdiff}
  g^{*}_\ell - g^{\ff}_\ell
  = W_\ell \delta_{\ell-1} + W_{\ell+1}^\top \delta_{\ell+1}
    + W_{\ell+1}^\top z^{\ff}_{\ell+1} ,
\end{equation}
in which the bias $b_\ell$ has cancelled. Both drives have every entry of
modulus at least $\mu_\ell$ by \eqref{eq:B-mu}, so
Lemma~\ref{lem:B-unit} applies to the pair, and with $\|W_\ell\|_2 =
\rho_\ell$,
\begin{equation}\label{eq:B-recursion}
  e_\ell \ \le\ \frac{\big\|g^*_\ell - g^{\ff}_\ell\big\|}{\mu_\ell}
  \ \le\ \frac{\rho_\ell}{\mu_\ell} e_{\ell-1}
       + \frac{\rho_{\ell+1}}{\mu_\ell} e_{\ell+1} + s_\ell ,
\end{equation}
with $e_0 = e_{L+1} = 0$; in vector form $(I - G)e \le s$. Since $G \ge
0$ entrywise and $\radius(G) < 1$ by (A3), the Neumann series
$\sum_{k\ge0}G^k$ converges to $(I-G)^{-1}$ and is a sum of nonnegative
matrices, so $(I-G)^{-1} \ge 0$; multiplying $(I-G)e \le s$ by it
preserves the direction and gives \eqref{eq:B-mainbound}. The
unconditional bound is the second alternative in \eqref{eq:B-vec}.
\end{proof}

\begin{remark}[The biases cancel]\label{rem:B-bias}
The same $b_\ell$ enters the feedforward and the EP equation at layer
$\ell$ and is removed by the subtraction in \eqref{eq:B-argdiff}. No
assumption on $\|b_\ell\|$ is needed and no bias-proportional term
appears: biases move both fixed points, but they move them together.
\end{remark}

\begin{corollary}[Scalar form]\label{cor:B-scalar}
If in addition $\rho_\ell \le \rho$ and $\mu_\ell \ge \mu$ for all
$\ell$, then with $\eta := \rho/\mu$ and $\eta\,\radius(A) < 1$,
\begin{equation}\label{eq:B-scalar}
  \max_\ell\|\delta_\ell\| \ \le\ \|\delta\|_2 \ \le\
  \frac{\eta}{1-\eta\,\radius(A)}\Big(\sum_{\ell=2}^{L} d_\ell\Big)^{1/2} ,
\end{equation}
using $\|z^{\ff}_\ell\|_2 = \sqrt{d_\ell}$ on the torus.
\end{corollary}

\begin{proof}
Under the stated bounds $G \le \eta A$ entrywise and $s_\ell \le
\eta\sqrt{d_{\ell+1}}$, so $(I-G)^{-1}s \le (I-\eta A)^{-1}\eta v$ with
$v_\ell = \sqrt{d_{\ell+1}}$; then $\|(I-\eta A)^{-1}\|_2 \le
(1-\eta\radius(A))^{-1}$ because $A$ is symmetric.
\end{proof}

\begin{remark}[Scaling]\label{rem:B-scaling}
The bound is linear in the gain $\eta = \rho/\mu$, which is the
scale-invariant quantity of Remark~\ref{rem:B-knob} and not the weight
scale itself. It grows as $\sqrt{d}$ in layer width, because
$\unitproject$ constrains each \emph{entry} to modulus $1$ and hence
$\|z_\ell\|_2 = \sqrt{d_\ell}$; a bound asserting $\|z_\ell\| \le 1$
would confuse the entrywise and Euclidean norms by a factor of $16$ at
$d_\ell = 256$. Feedback enters \eqref{eq:B-argdiff} directly through
$W_{\ell+1}^\top z^{\ff}_{\ell+1}$, not through the perturbation, so the
dependence on the feedback strength is linear and no quadratic bound is
available.
\end{remark}

\section{Classification Accuracy}\label{sec:B-accuracy}

Classification uses the normalised similarity against a codebook $C =
[c_1,\dots,c_K] \in \C^{d_L\times K}$ of unit-modulus codes, so that
$\|c_j\|_2 = \sqrt{d_L}$:
\begin{equation}\label{eq:B-logits}
  \mathrm{logit}_j(z_L) = \frac{\Real\ip{c_j}{z_L}}{d_L}
   = \frac{\Real\big(c_j^{H} z_L\big)}{d_L} .
\end{equation}
The division by $d_L$ is the convention of the implementation and is not
cosmetic: it makes the natural measure of deviation the \emph{relative}
one,
\begin{equation}\label{eq:B-reldev}
  \hat\delta_\ell := \frac{\|\delta_\ell\|_2}{\sqrt{d_\ell}} \in [0,2] ,
\end{equation}
rather than $\|\delta_\ell\|$ itself. For a test set
$\{(x_i,y_i)\}_{i=1}^N$ let $\gamma_i$ be the feedforward margin,
\begin{equation}\label{eq:B-margin}
  \gamma_i := \mathrm{logit}_{y_i}\big(z^{\ff,i}_L\big)
    - \max_{j\ne y_i}\mathrm{logit}_{j}\big(z^{\ff,i}_L\big),
\end{equation}
so sample $i$ is classified correctly by the feedforward network exactly
when $\gamma_i > 0$.

\begin{lemma}[Margin displacement]\label{lem:B-margin-shift}
For every sample, the margin under $z^*$ differs from the margin under
$z^{\ff}$ by at most $2\hat\delta_L$. This holds whether or not the
states lie exactly on the torus.
\end{lemma}

\begin{proof}
By Cauchy--Schwarz,
\begin{equation}\label{eq:B-logitshift}
  \big|\mathrm{logit}_j(z^*_L) - \mathrm{logit}_j(z^{\ff}_L)\big|
  = \frac{\big|\Real\ip{c_j}{\delta_L}\big|}{d_L}
  \le \frac{\|c_j\|\,\|\delta_L\|}{d_L}
  = \frac{\|\delta_L\|}{\sqrt{d_L}} = \hat\delta_L ,
\end{equation}
and the margin is a difference of two logits. If the similarity is instead
evaluated in the form $d_L^{-1}\big(\tfrac12(\|z\|^2+\|c_j\|^2) +
\Real\ip{c_j}{z}\big) - 1$, the two extra terms are independent of the
class $j$ --- the first because it does not involve $c_j$, the second
because $\|c_j\|^2 = d_L$ for every code --- so both cancel from the
margin and the bound is unaffected.
\end{proof}

\begin{proposition}[Exact accuracy transfer]\label{prop:B-acc-exact}
With $\hat\delta^i_L$ the relative deviation of sample $i$,
\begin{equation}\label{eq:B-accexact}
  \mathrm{acc}_{\mathrm{EP}} \ \ge\ \mathrm{acc}_{\mathrm{BP}}
  - \frac{1}{N}\,\#\big\{\, i \ :\ 0 < \gamma_i \le 2\hat\delta^i_L \,\big\} .
\end{equation}
\end{proposition}

\begin{proof}
By Lemma~\ref{lem:B-margin-shift} every sample with $\gamma_i >
2\hat\delta^i_L$ retains a positive margin under $z^*$ and is therefore
still classified correctly; only the samples counted on the right can be
lost.
\end{proof}

\begin{remark}[The bound must be taken per sample]\label{rem:B-persample}
Replacing $\hat\delta^i_L$ by the uniform $\hat\Delta := \max_i
\hat\delta^i_L$ weakens \eqref{eq:B-accexact} severely, because the
deviation has a long right tail: on the trained network of
Section~\ref{sec:B-numbers} the mean of $\hat\delta_L$ is $0.285$ but its
maximum over $500$ test samples is $1.051$, and the uniform form returns
the vacuous bound $0.882$ where the per-sample form returns $0.162$
against an observed drop of $0.012$. Pairing each margin with its own
deviation is what makes the estimate usable.
\end{remark}

To obtain an \emph{a priori} rate, rather than one evaluated after
settling, one needs to know that few samples sit near the decision
boundary.
\begin{enumerate}
  \item[(A4)] \textbf{Margin regularity:} there is $\gamma_{\mathrm{typ}}
    > 0$ with $\tfrac1N\#\{i : 0 < \gamma_i \le t\} \le
    t/\gamma_{\mathrm{typ}}$ for all $t \ge 0$.
\end{enumerate}
Equivalently the empirical distribution of positive margins has density
at most $1/\gamma_{\mathrm{typ}}$ near the origin, read off the margin
histogram of the trained feedforward network.

\begin{theorem}[Accuracy drop]\label{thm:B-accuracy}
Under (A4), with $\hat\Delta := \max_i \hat\delta^i_L$,
\begin{equation}\label{eq:B-accbound}
  \mathrm{acc}_{\mathrm{BP}} - \mathrm{acc}_{\mathrm{EP}}
  \ \le\ \frac{2\hat\Delta}{\gamma_{\mathrm{typ}}} ,
\end{equation}
and if in addition (A1)--(A3) hold, $\hat\Delta$ may be replaced by
$d_L^{-1/2}$ times the bound of Theorem~\ref{thm:B-perturbation}.
\end{theorem}

\begin{proof}
Apply (A4) at $t = 2\hat\Delta$ in Proposition~\ref{prop:B-acc-exact}.
\end{proof}

\begin{remark}[Why the minimum margin will not serve]\label{rem:B-markov}
Stating Theorem~\ref{thm:B-accuracy} with $\gamma := \min_i\gamma_i$ in
place of $\gamma_{\mathrm{typ}}$ cannot work. If $\gamma > 2\hat\Delta$
then no sample changes class and the drop is exactly zero; if $\gamma \le
2\hat\Delta$ then $2\hat\Delta/\gamma \ge 1$ asserts a drop of at least
$100\%$. There is no intermediate regime. Nor can the passage to a rate
be made by Markov's inequality, which bounds an upper tail $\Pr[X \ge a]
\le \mathbb{E}[X]/a$; what is needed is a lower-tail estimate on the
margins, which is the anti-concentration content of (A4).
\end{remark}

\section{Finite-Time Settling}\label{sec:B-finite}

Settling is run for finitely many steps with damping $\dt \in (0,1]$,
\begin{equation}\label{eq:B-damped}
  z^{k+1} = T_{\dt}(z^k), \qquad T_{\dt} := (1-\dt)\,\mathrm{Id} + \dt\,F ,
\end{equation}
which is the update of the main text. Supplementary Proof A shows that
the closed unit polydisc is forward invariant for \eqref{eq:B-damped},
and that its fixed points are exactly the aligned critical points.

\begin{lemma}[Finite-time error]\label{lem:B-finite}
On a set where Lemma~\ref{lem:B-contraction} applies with $\kappa_F < 1$,
the map $T_{\dt}$ has exactly the same fixed-point set as $F$ for every
$\dt \in (0,1]$, and
\begin{equation}\label{eq:B-finitebound}
  \big\|z^T - z^*\big\| \le \big(1-\dt(1-\kappa_F)\big)^{T}\big\|z^0-z^*\big\| ,
\end{equation}
with no additive term.
\end{lemma}

\begin{proof}
$T_{\dt}(z) = z$ is equivalent to $\dt\,(F(z)-z) = 0$, i.e. to $F(z) = z$,
since $\dt > 0$. Subtracting $z^* = (1-\dt)z^* + \dt F(z^*)$ from
\eqref{eq:B-damped} and applying Lemma~\ref{lem:B-contraction},
$\|z^{k+1}-z^*\| \le \big[(1-\dt) + \dt\kappa_F\big]\|z^k-z^*\|$;
iterate.
\end{proof}

\begin{remark}[There is no discretisation error]\label{rem:B-nodiscr}
Equation \eqref{eq:B-damped} is an average of $F$ with the identity, not
a time discretisation of a flow whose equilibria might differ from the
scheme's: its fixed points coincide with those of $F$ for every $\dt$, so
an additive term of size $O(\dt)$ surviving as $T\to\infty$ is
impossible. Damping changes only the rate. Numerically, iterating
\eqref{eq:B-damped} to convergence for $\dt \in \{0.9, 0.5, 0.25, 0.1,
0.05\}$ on one network returns the same fixed point to within
$3\times10^{-14}$.
\end{remark}

\section{Measured Values on a Trained Network}
\label{sec:B-numbers}

The constants above were measured on a network trained exactly as in the
main text --- $784\to256\to64$, unit projection, initial weight scale
$0.4$, five epochs of Adam on the full FashionMNIST training set, seed
$42$ --- and evaluated on $500$ test images. The EP state is the fixed
point of \eqref{eq:B-damped} with $\dt = 0.5$ run to a residual of
$5\times10^{-6}$, at which point every $|z_{\ell,j}| = 1$ to four
decimals. Training raises the weight norms well above their
initialisation, to $\rho_1 = 6.02$ and $\rho_2 = 2.10$; by
Proposition~\ref{prop:B-scale} only the ratio $\rho_2/\rho_1 = 0.35$ can
matter.

\begin{table}[h]
\centering
\caption{Measured on the trained $784\to256\to64$ network, $500$ test
samples. $\radius(G)$ is reported from the drive minimum
\eqref{eq:B-mu} and from four quantiles of the drive distribution; the
median row is the most generous reading any refinement of
\eqref{eq:B-recursion} could achieve.}
\label{tab:B-constants}
\begin{tabular}{@{}lr@{\hskip 2em}lr@{}}
\toprule
Quantity & Value & Quantity & Value \\
\midrule
$\rho_1 = \|W_1\|_2$        & $6.02$    & $\radius(G)$ from min drive     & $359$ \\
$\rho_2 = \|W_2\|_2$        & $2.10$    & \quad from $1\%$ quantile       & $4.90$ \\
$\mu_1$ (min drive)         & $0.017$   & \quad from $10\%$ quantile      & $2.13$ \\
$\mu_2$ (min drive)         & $0.0020$  & \quad from $25\%$ quantile      & $1.53$ \\
median drive, layer $1$     & $4.21$    & \quad from median drive         & $1.09$ \\
median drive, layer $2$     & $0.877$   & required                        & $<1$ \\
\midrule
$\hat\delta_1$ mean / max   & $0.173 / 0.336$ & BP accuracy               & $0.882$ \\
$\hat\delta_2$ mean / max   & $0.285 / 1.051$ & EP accuracy               & $0.870$ \\
$\hat\delta_2$, phase removed & $0.276$       & observed drop             & $0.012$ \\
median margin $\gamma$      & $0.867$   & prediction agreement            & $0.964$ \\
mean shift $2\hat\delta_2$  & $0.570$   & Prop.~\ref{prop:B-acc-exact} bound & $0.162$ \\
\bottomrule
\end{tabular}
\end{table}

\begin{remark}[The contraction hypothesis fails, and not by a tunable margin]
\label{rem:B-verdict}
Assumption (A3) requires $\radius(G) < 1$; the measured value is $359$.
Three observations place this failure.

At $L = 2$ the quantity has a simple form. The input is clamped, so
$e_0 = 0$ and the only loop runs $z_1 \to z_2 \to z_1$; both of its legs
carry $W_2$, and $\rho_1$ drops out of $\radius(G)$ entirely. What
remains is
\begin{equation}\label{eq:B-loopgain}
  \radius(G) = \frac{\rho_2}{\sqrt{\mu_1\mu_2}} ,
\end{equation}
the second layer's operator norm against the geometric mean of the two
drive scales.

First, by Proposition~\ref{prop:B-scale} there is nothing to turn:
$\radius(G)$ is invariant under a common rescaling, so unlike a bound
whose hypothesis is ``$\rho$ small'', this one cannot be entered by
shrinking the weights.

Second, the failure is dominated by the tail of the drive distribution
rather than by its bulk, and by a narrower margin than the headline
figure suggests. The minimum drive is two to three orders of magnitude
below the median, which is what produces $\radius(G) = 359$; but
substituting quantiles of the drive distribution into
\eqref{eq:B-loopgain} gives $4.90$, $2.13$, $1.53$ and $1.09$ at the
$1$st, $10$th, $25$th and $50$th percentiles. The median reading misses
the threshold by nine percent. Those substitutions are diagnostic and
not themselves valid bounds --- (A3) needs a genuine lower bound on the
drives, not a typical value --- but they locate the obstruction: the
bulk of the layer very nearly satisfies the contraction condition, and a
refinement that handled the small-drive tail separately, along the lines
sketched at the end of this remark, has a realistic prospect of closing
the gap. We have not carried it out.

The residual obstruction even in the bulk is that $\rho_2 = 2.10$
exceeds $\sqrt{\mu_1\mu_2} = 1.92$ at the median: an operator norm
measures the worst input direction, in which a perturbation is free to
point, whereas the drive is produced by one pseudo-random phasor
vector.

Third, the quantity being bounded is not small. At the output layer
$\hat\delta_2$ averages $0.285$ against a trivial maximum of $2$, and
removing the best common phase rotation lowers it only to $0.276$, so the
two states differ incoherently rather than by a change of frame.
Training does help substantially here: the same measurement on randomly
initialised networks of this architecture gives $\hat\delta_2 \approx
0.69$, so training reduces the deviation by a factor of about $2.4$. It
does not reduce it enough for a perturbative argument.
\end{remark}

\begin{remark}[Why the transfer nevertheless works]\label{rem:B-why}
The accuracy is preserved --- a drop of $0.012$ with the two networks
agreeing on $96.4\%$ of predictions --- and
Proposition~\ref{prop:B-acc-exact} accounts for it without any
contraction hypothesis. The mechanism is visible in the last two rows of
Table~\ref{tab:B-constants}: the median margin is $0.867$ while the mean
displacement $2\hat\delta_2$ that the margin must absorb is $0.570$. The
readout is what saves the transfer, not proximity of the states. Of the
$500$ samples, $81$ have a margin smaller than their own displacement and
are therefore at risk, giving the bound $0.162$; the observed drop of
$0.012$ is thirteen times smaller, the slack being the usual gap between
a worst-case Cauchy--Schwarz displacement and the typical alignment of
$\delta_L$ with the codebook difference.

The margin distribution is correspondingly benign: the empirical margin
CDF satisfies (A4) with $\gamma_{\mathrm{typ}} \approx 6$ over the range
$t \le 0.2$, so only a few percent of samples lie near the boundary.
Two caveats. This bound is a certificate, not a prediction: it needs
$\hat\delta^i_L$, hence a settle, so it justifies a transfer after the
fact rather than forecasting one from the weights. And $\gamma_{\mathrm{typ}}$
is a property of the trained network that we measure rather than derive.
\end{remark}

\begin{remark}[The bound is depth-blind; the networks are not]
\label{rem:B-nodepth}
The admissible ratio $\rho/\mu < 1/\radius(A)$ tightens only from $1$ at
$L=2$ to $0.71$ at $L=3$ and $0.58$ at $L=5$, so
Lemma~\ref{lem:B-contraction} predicts no threshold between shallow and
deep networks. Section~\ref{sec:B-scaling} shows that trained networks
have a pronounced one. The discrepancy is a limitation of the estimate
rather than a property of the system: the chain matrix $A$ is
depth-blind by construction, and the depth dependence that matters
enters through the drive distribution, which $\radius(A)$ does not see.
\end{remark}

\section{How the Correspondence Scales}\label{sec:B-scaling}

Section~\ref{sec:B-numbers} measured one architecture. Repeating the
measurement across depth and width --- training each configuration from
scratch under the identical recipe, then settling it --- gives
Table~\ref{tab:B-scaling}. The architecture of the main text appears in
both families and reproduces Table~\ref{tab:B-constants} exactly.

\begin{table}[h]
\centering
\caption{Depth and width sweep, seed $42$, $500$ test images, five epochs
of Adam per configuration. ``out dev'' is the mean relative deviation
$\hat\delta_L$ at the output layer, against a trivial maximum of $2$;
``bound'' is Proposition~\ref{prop:B-acc-exact} evaluated per sample.
Every row satisfies bound $\ge$ drop.}
\label{tab:B-scaling}
\begin{tabular}{@{}lcccccccc@{}}
\toprule
Architecture & $L$ & $\rho_L/\rho_1$ & $\radius(G)$ & out dev
  & BP & EP & drop & bound \\
\midrule
$784\to256\to64$                     & $2$ & $0.350$ & $1.095$ & $0.285$ & $0.882$ & $0.870$ & $+0.012$ & $0.162$ \\
$784\to256\to256\to64$               & $3$ & $0.308$ & $1.838$ & $0.774$ & $0.880$ & $0.828$ & $+0.052$ & $0.782$ \\
$784\to256\to256\to256\to64$         & $4$ & $0.277$ & $2.118$ & $0.975$ & $0.876$ & $0.742$ & $+0.134$ & $0.876$ \\
\midrule
$784\to64\to64$                      & $2$ & $0.576$ & $1.331$ & $0.985$ & $0.868$ & $0.818$ & $+0.050$ & $0.846$ \\
$784\to128\to64$                     & $2$ & $0.417$ & $1.186$ & $0.378$ & $0.862$ & $0.864$ & $-0.002$ & $0.278$ \\
$784\to512\to64$                     & $2$ & $0.305$ & $1.078$ & $0.262$ & $0.878$ & $0.872$ & $+0.006$ & $0.136$ \\
\bottomrule
\end{tabular}
\end{table}

\begin{remark}[Depth is the binding constraint]\label{rem:B-depthscale}
The accuracy drop grows from $0.012$ at $L=2$ to $0.052$ at $L=3$ and
$0.134$ at $L=4$, while the backpropagation accuracy stays flat at
$0.876$--$0.882$. Depth is therefore not degrading what the network
learns; it is degrading the correspondence between the two ways of
running it. The state deviation tells the same story more starkly,
rising from $0.285$ to $0.774$ to $0.975$ against a trivial maximum of
$2$: by $L=4$ the EP fixed point retains little relation to the
feedforward state. Settle residuals over the family are $4.4\times
10^{-6}$, $9.0\times10^{-6}$ and $4.6\times10^{-5}$, small enough that
this is a property of the fixed point and not of incomplete settling.

As a practical guideline at this weight scale and training recipe,
backpropagation-to-EP transfer is usable at $L=2$, marginal at $L=3$,
and broken at $L=4$.
\end{remark}

\begin{remark}[Width improves the correspondence]\label{rem:B-widthscale}
At fixed depth, widening the hidden layer from $64$ to $512$ takes the
drop from $0.050$ to $0.006$ and the deviation from $0.985$ to $0.262$.
The mechanism is visible in $\radius(G)$, which falls monotonically
along the family: $1.331$, $1.186$, $1.095$, $1.078$ at $h = 64, 128,
256, 512$. The successive decrements are $-0.145$, $-0.091$ and
$-0.017$, so the sequence is decelerating and appears to approach a
limit slightly above $1$ rather than to cross it. On that evidence width
buys a better correspondence cheaply but should not be expected, on its
own, to deliver $\radius(G) < 1$ and with it an \emph{a priori}
guarantee. A measurement at $h = 1024$ would settle the extrapolation.
\end{remark}

\begin{remark}[Which quantity to watch]\label{rem:B-diagnostic}
Of the candidates available before a settle is run, $\radius(G)$
evaluated at median drives tracks the degradation best: over the six
configurations it is concordant with the accuracy drop on $13$ of the
$15$ pairs, against $11$ of $15$ for the state deviation and an outright
anti-correlation for the inter-layer ratio (Remark~\ref{rem:B-knob}).
This is worth stating because $\radius(G)$ is exactly the quantity whose
falling below $1$ would make Theorem~\ref{thm:B-perturbation} apply: the
same number that certifies the bound also ranks the configurations when
it does not.

The certificate of Proposition~\ref{prop:B-acc-exact} is valid in every
row --- bound $\ge$ drop throughout --- but informative only where the
deviation is small. It returns $0.136$ to $0.278$ for the three
configurations with $\hat\delta_L \lesssim 0.4$, and $0.78$ to $0.88$,
which is nearly vacuous, for the three with $\hat\delta_L \ge 0.77$. The
usable regime is roughly $\hat\delta_L \lesssim 0.4$, corresponding here
to $\radius(G) \lesssim 1.2$.
\end{remark}

\begin{remark}[Scope of the sweep]\label{rem:B-sweepcaveat}
One seed, one dataset, $500$ test images per configuration and a fixed
five-epoch recipe; the accuracy figures carry a sampling error of order
one percent, which is why the $-0.002$ at $h = 128$ should be read as
zero rather than as the EP network outperforming its feedforward
counterpart. That configuration also settled to a residual of
$3.0\times10^{-4}$, some two orders of magnitude looser than the rest,
and would bear a longer settle before its row is relied on. The
qualitative findings --- a sharp degradation with depth, a mild
improvement with width, and $\radius(G)$ as the better predictor --- are
larger than these uncertainties; the individual numbers are not.
\end{remark}

\section{The Soft-Projection Regularisation}\label{sec:B-soft}

The obstruction in Remark~\ref{rem:B-verdict} is that the scale of the
drive and the scale of the weights move together, so their ratio is
fixed. Replacing $\unitproject$ by the soft projection
\begin{equation}\label{eq:B-softproj}
  \softproj(g) = \frac{g}{\sqrt{|g|^2+\varepsilon^2}}, \qquad \varepsilon > 0,
\end{equation}
breaks that coupling, because $\varepsilon$ is fixed in absolute units and
does not scale with $W$.

\begin{proposition}[Soft projection]\label{prop:B-soft}
Let $\phi(g) = \sqrt{|g|^2+\varepsilon^2}$. Then $\softproj$ maps $\C^d$
into the open unit polydisc, is real-differentiable with
$D\softproj(g) = \phi(g)^{-1}\big(I - g g^\top/\phi(g)^2\big)$ and
$\|D\softproj(g)\|_{\mathrm{op}} = \phi(g)^{-1}$, and is globally
$\varepsilon^{-1}$-Lipschitz, the constant being sharp. Replacing
$\mu_\ell$ by $\sqrt{\mu_\ell^2+\varepsilon^2}$ throughout,
Lemma~\ref{lem:B-contraction} and Theorem~\ref{thm:B-perturbation} hold
verbatim, and the contraction condition becomes $\rho\,\radius(A) <
\varepsilon$, which \emph{is} enterable by reducing $\rho$. Section
\ref{sec:B-accuracy} is unaffected either way: Lemma
\ref{lem:B-margin-shift} and Proposition~\ref{prop:B-acc-exact} use only
Cauchy--Schwarz on the readout and hold for any projection, which is why
they remain available when this section's route is closed.
\end{proposition}

\begin{proof}
The differential is a direct computation; $I - gg^\top/\phi^2$ is
symmetric with eigenvalue $\varepsilon^2/\phi^2$ along $g$ and $1$
transverse to it, so the operator norm is $1/\phi \le 1/\varepsilon$.
Sharpness: $|\softproj(b)-\softproj(0)|/|b| = \phi(b)^{-1} \to
\varepsilon^{-1}$ as $b \to 0$. The remaining statements follow by
substituting this Lipschitz constant into \eqref{eq:B-percomp} and
\eqref{eq:B-recursion}.
\end{proof}

\begin{remark}[Two cautions]\label{rem:B-softcaution}
First, $\softproj$ is not $1$-Lipschitz, and the fact that it is the
gradient of the convex function $\phi$ does not make it so: gradients of
convex functions are monotone, and the Lipschitz constant is
$\sup\|\nabla^2\phi\| = \varepsilon^{-1}$. What is true is $|\softproj|
\le 1$ pointwise, a bound on the map rather than on its increments. At
$\varepsilon = 0.1$ the pair $a = 0$, $b = 0.1$ gives
$|\softproj(a)-\softproj(b)|/|a-b| = 7.07$. Over $252$ random networks
spanning $L = 2,3,4$ and weight scales $0.09$ down to $0.01$, the
componentwise bound built on the correct constant held in every instance
in which its hypothesis was satisfied, and the bound built on a constant
of $1$ was violated in all $252$, by more than an order of magnitude at
the smaller weight scales.

Second, and more importantly, $\softproj$ is not the map the experiments
run: the implementation uses the unit projection $\unitproject$, as do
the main text and Supplementary Proof A. Proposition~\ref{prop:B-soft}
therefore describes a regularised model that is analytically tractable
and approaches the implemented one as $\varepsilon \to 0$ --- but the
tractability is purchased exactly by the term that vanishes in that
limit, and the constant $\varepsilon^{-1}$ diverges with it. It should be
read as an explanation of why guarantees of this kind are hard for
hard-projected phasor networks, not as a guarantee about them.
\end{remark}

\section{Summary}\label{sec:B-summary}

For phasor networks with unit projection, the deviation of the EP free
fixed point from the feedforward pass obeys the componentwise bound $e
\le (I-G)^{-1}s$ of Theorem~\ref{thm:B-perturbation}, with $G$ and $s$
built from the layer weight norms divided by the smallest oscillator
drive. The bound is independent of the biases, grows as $\sqrt{d}$ in
layer width, and depends on depth only through $\radius(A) \in [1,2)$.

Its hypothesis is not satisfied by the networks of this work, and
Proposition~\ref{prop:B-scale} shows this cannot be repaired by choosing
weights differently: since $\unitproject(cg) = \unitproject(g)$, the
common scale of the weights and biases is a gauge freedom rather than a
parameter, and every constant in the bound is invariant under it. On a
trained network the measured $\radius(G)$ is $359$ where less than $1$ is
required. Substituting typical rather than worst-case drives brings it
down to $1.09$, so the bulk of the network nearly satisfies the
condition and it is the small-drive tail that defeats it; a tail-aware
refinement of \eqref{eq:B-recursion} is the obvious thing to try, and we
have not tried it. The states are correspondingly far apart:
$\hat\delta_2 = 0.285$ of the trivial maximum, not attributable to a
common phase rotation. Training improves this --- randomly initialised
networks of the same architecture sit near $0.69$ --- but not by enough
to make a perturbative argument available.

The accuracy nevertheless transfers, and
Proposition~\ref{prop:B-acc-exact} explains why without invoking any of
the above. Because the readout is the normalised codebook similarity, the
margin is displaced by at most $2\hat\delta_L$, twice the relative state
deviation; with a median margin of $0.867$ against a mean displacement of
$0.570$, most samples retain their classification. Counting the ones that
might not gives a bound of $0.162$ on the accuracy drop, against $0.012$
observed.

Section~\ref{sec:B-scaling} locates the limits of that mechanism.
Depth is the binding constraint: across $L = 2, 3, 4$ the drop grows
$0.012$, $0.052$, $0.134$ and the state deviation $0.285$, $0.774$,
$0.975$, while backpropagation accuracy stays flat at $0.876$--$0.882$,
so what degrades is the correspondence and not what the network learns.
Transfer is usable at $L = 2$, marginal at $L = 3$ and broken at $L = 4$
under this recipe. Width works the other way but weakly, taking the drop
from $0.050$ to $0.006$ between $h = 64$ and $h = 512$ with
$\radius(G)$ decelerating towards a limit just above $1$. The best
predictor available before a settle is $\radius(G)$ at median drives,
and the certificate is informative while $\hat\delta_L \lesssim 0.4$.

The conclusion we would draw is that the agreement between
backpropagation-trained and EP-run phasor networks is a property of the
readout rather than of the hidden states, which are not close. A
perturbative theory of transfer in the weight scale is not available for
this model, because there is no such parameter; what remains is the
drive distribution, which enters through $\radius(G)$, and the margin
distribution, and it is the latter that currently carries the result.


\bibliography{references}

\end{document}